\documentclass{article}

\PassOptionsToPackage{numbers, compress, comma}{natbib}

 \usepackage[final]{neurips_2020}

\usepackage[utf8]{inputenc} %
\usepackage[T1]{fontenc}    %
\usepackage{xcolor}         %
\definecolor{mydarkblue}{rgb}{0,0.08,0.45}
\usepackage[colorlinks,citecolor=mydarkblue,urlcolor=mydarkblue]{hyperref}

\usepackage{url}            %
\usepackage{booktabs}       %
\usepackage{amsfonts}       %
\usepackage{nicefrac}       %
\usepackage{microtype}      %
\usepackage{amsmath, bm, amsthm}
\usepackage{thmtools,thm-restate}

\newtheorem{corollary}{Corollary}
\newcommand\norm[1]{\left\lVert#1\right\rVert}%
\newcommand\given[1][]{\:#1\vert\:}

\usepackage{subcaption}
\usepackage{graphicx}
\usepackage{algorithm}
\usepackage{algorithmic}
\usepackage{makecell}
\usepackage{todonotes}
\usepackage[hang,flushmargin]{footmisc}

 \usepackage[compact]{titlesec}
 \titlespacing{\section}{0pt}{1ex}{0ex}
 \titlespacing{\subsection}{0pt}{1ex}{0ex}
 \titlespacing{\subsubsection}{0pt}{0.5ex}{0ex}
 \newcommand{\reducespaceafterfigure}{\vspace{-0.5em}} %
\usepackage{appendix}
\usepackage[capitalise, noabbrev]{cleveref}
\crefmultiformat{section}{Sections~#2#1#3}{ \&~#2#1#3}{,~#2#1#3}{ \&~#2#1#3}
\crefformat{equation}{Eq.~(#2#1#3)}

\newcommand{\NTKGP}{NTKGP}

\newcommand{\minus}{\mkern1.5mu{-}\mkern1.5mu}
\newcommand{\plus}{\mkern1.5mu{+}\mkern1.5mu}
\newcommand{\myeq}{\mkern1.5mu{=}\mkern1.5mu}
\newcommand{\myin}{\mkern1.5mu{\in}\mkern1.5mu}
\newcommand{\mysim}{\mkern1.5mu{\overset{d}{\sim}}\mkern1.5mu}
\title{Bayesian Deep Ensembles via the \\
Neural Tangent Kernel}

\usepackage{amsmath,amsfonts,bm}

\def\eqref#1{equation~\ref{#1}}

\def\1{\bm{1}}
\newcommand{\train}{\mathcal{D}}

\newcommand{\ThetaXX}{\Theta_{\mathcal X \mathcal X}}
\newcommand{\Thetaxy}[2]{\Theta_{#1 #2}}

\newcommand{\KXX}{\mathcal K_{\mathcal X \mathcal X}}
\newcommand{\K}{\mathcal K}
\newcommand{\X}{\mathcal X}
\newcommand{\Y}{\mathcal Y}
\def\eps{{\epsilon}}

\def\vtheta{{\bm{\theta}}}

\def\vx{{\bm{x}}}

\def\vz{{\bm{z}}}

\DeclareMathAlphabet{\mathsfit}{\encodingdefault}{\sfdefault}{m}{sl}
\SetMathAlphabet{\mathsfit}{bold}{\encodingdefault}{\sfdefault}{bx}{n}

\author{%
  Bobby He\\ %
  Department of Statistics\\
  University of Oxford\\
  \texttt{bobby.he@stats.ox.ac.uk} \\
  \And
  Balaji Lakshminarayanan \\
  Google Research\\ Brain team \\
  \texttt{balajiln@google.com}
  \And
  Yee Whye Teh \\
  Department of Statistics \\
  University of Oxford \\
  \texttt{y.w.teh@stats.ox.ac.uk}
}

\begin{document}
\maketitle
\begin{abstract}
We explore the link between deep ensembles and Gaussian processes (GPs) through the lens of the Neural Tangent Kernel (NTK): a recent development in understanding the training dynamics of wide neural networks (NNs). Previous work has shown that even in the infinite width limit, when NNs become GPs, there is no GP posterior interpretation to a deep ensemble trained with squared error loss. We introduce a simple modification to standard deep ensembles training, through addition of a computationally-tractable, randomised and untrainable function to each ensemble member, that enables a posterior interpretation in the infinite width limit. When ensembled together, our trained NNs give an approximation to a posterior predictive distribution, and we prove that our Bayesian deep ensembles make more conservative predictions than standard deep ensembles in the infinite width limit. Finally, using finite width NNs we demonstrate that our Bayesian deep ensembles faithfully emulate the analytic posterior predictive when available, and can outperform standard deep ensembles in various out-of-distribution settings, for both regression and classification tasks.
\end{abstract}
\section{Introduction}\label{sec:intro}
Consider a training dataset $\mathcal{D}$ consisting of $N$ i.i.d. data points $\mathcal{D} = \{\mathcal{X}, \mathcal{Y}\}= \{(\vx_n, y_n)\}_{n=1}^N$, with $\vx\myin\mathbb{R}^d$ representing $d$-dimensional features and $y$ representing $C$-dimensional targets.  Given input features $\vx$ and parameters $\vtheta\myin\mathbb{R}^p$ we use the output, $f(\vx, \vtheta)\myin\mathbb{R}^C$, of a neural network (NN) to model the predictive distribution $p(y|\vx, \vtheta)$ over the targets. For univariate regression tasks, $p(y|\vx, \vtheta)$ will be Gaussian: $-\textrm{log}\hspace{0.5mm}p(y|\vx, \vtheta)$ is the squared error $\frac{1}{2\sigma^2}(y - f(\vx, \vtheta))^2$ up to additive constant, for fixed observation noise $\sigma^2\myin\mathbb{R}_+$ . For classification tasks, $p(y|\vx, \vtheta)$ will be a Categorical distribution.

Given a prior distribution $p(\vtheta)$ over the parameters, we can define the posterior over $\vtheta$, $p(\vtheta|\train)$, using Bayes' rule and subsequently the \textit{posterior predictive} distribution at a test point $(\vx^*, y^*)$:
\begin{align}\label{eqn:posterior_predictive}
    p(y^*|\vx^*, \train)=\int p(y^*|\vx^*, \vtheta) p(\vtheta|\train)\, d\vtheta
\end{align}
The posterior predictive is appealing as it represents a marginalisation over $\vtheta$ weighted by posterior probabilities, and has been shown to be optimal for minimising predictive risk under a well-specified model \citep{aitchison1975goodness}. However, one issue with the posterior predictive for NNs is that it is computationally intensive to calculate the posterior $p(\vtheta|\train)$ exactly. Several approximations to $p(\vtheta|\train)$ have been introduced for \textit{Bayesian neural networks} (BNNs) including: Laplace approximation \citep{mackay1992bayesian}; Markov chain Monte Carlo \citep{neal2012bayesian, welling2011bayesian}; variational inference \citep{graves2011practical, blundell2015weight, louizos2017multiplicative,flipout, sun2019functional}; and Monte-Carlo dropout \citep{gal2016dropout}.

Despite the recent interest in BNNs, it has been shown empirically that deep ensembles \citep{balajideepensembles}, which lack a principled Bayesian justification, outperform existing BNNs in terms of uncertainty quantification and out-of-distribution robustness, cf. \citep{ovadia2019can}. Deep ensembles independently initialise and train individual NNs (referred to herein as \textit{baselearners}) on the negative log-likelihood loss $\mathcal{L}(\vtheta) \myeq \sum_{n=1}^N \ell(y_n, f(\vx_n, \vtheta))$
 with $\ell(y, f(\vx,{\vtheta}))\myeq\minus\textrm{log}\hspace{0.5mm}p(y|\vx, \vtheta)$, before aggregating predictions. Understanding the success of deep ensembles, particularly in relation to Bayesian inference, is a key question in the uncertainty quantification and Bayesian deep learning communities at present: \citet{fort2019deep} suggested that the empirical performance of deep ensembles is explained by their ability to explore different functional modes, while \citet{wilson2020bayesian} argued that deep ensembles are actually approximating the posterior predictive.

In this work, we will relate deep ensembles to Bayesian inference, using recent developments connecting GPs and wide NNs, both before \citep{neal1996priors, lee2018deep, alexander2018matthews, garriga-alonso2018deep, novak2019bayesian, yang2019tensori, hron2020infinite} and after \citep{jacot2018neural, lee2019wide} training. Using these insights, we  devise a modification to standard NN training that yields an exact posterior sample for $f(\cdot, \vtheta)$ in the infinite width limit. As a result, when ensembled together our modified baselearners give a posterior predictive approximation, and can thus be viewed as a \textit{Bayesian deep ensemble}.

 One concept that is related to our methods concerns ensembles trained with \textit{Randomised Priors} to give an approximate posterior interpretation, which we will use when modelling observation noise in regression tasks. The idea behind randomised priors is that, under certain conditions, regularising baselearner NNs towards independently drawn ``priors'' during training produces exact posterior samples for $f(\cdot, \vtheta)$. Randomised priors recently appeared in machine learning applied to reinforcement learning \citep{osband2018randomized} and uncertainty quantification
 \citep{pearce2018uncertainty, Ciosek2020Conservative}, like this work. To the best of our knowledge, related ideas first appeared in astrophysics where they were applied to Gaussian random fields \citep{hoffman1991constrained}. However, one such condition for posterior exactness with randomised priors is that the model $f(\vx, \vtheta)$ is linear in $\vtheta$. This is not true in general for NNs, but has been shown  to hold for wide NNs local to their parameter initialisation, in a recent line of work. In order to introduce our methods, we will first review this line of work, known as the \textit{Neural Tangent Kernel} (NTK) \citep{jacot2018neural}.
\section{NTK Background}\label{section:ntk}
Wide NNs, and their relation to GPs, have been a fruitful area recently for the theoretical study of NNs: we review only the most salient developments to this work, due to limited space.

First introduced by \citet{jacot2018neural}, the \textit{empirical NTK} of $f(\cdot, \vtheta_t)$ is, for inputs $\vx, \vx'$, the kernel:
\begin{align}
\hat\Theta_{t}(\vx, \vx')=\langle\nabla_{\vtheta}f(\vx, \vtheta_t), \nabla_{\vtheta}f(\vx', \vtheta_t)\rangle
\end{align}

and describes the functional gradient of a NN in terms of the current loss incurred on the training set. Note that $\vtheta_t$ depends on a random initialisation $\vtheta_0$, thus the empirical NTK is random for all $t>0$.

\citet{jacot2018neural}
showed that for an MLP under a so-called NTK parameterisation, detailed in \cref{app:parameterisations}, the empirical NTK converges in probability to a deterministic limit $\Theta$, that stays constant during gradient training, as the hidden layer widths of the NN go to infinity sequentially. Later, \citet{yang2019scaling, yang2020tensorii} extended the NTK convergence result to convergence almost surely, which is proven rigorously for a variety of architectures and for widths (or channels in Convolutional NNs) of hidden layers going to infinity in unison. This limiting positive-definite (p.d.) kernel $\Theta$, known as the NTK, depends only on certain NN architecture choices, including: activation, depth and variances for weight and bias parameters. Note that the NTK parameterisation can be thought of as akin to training under standard parameterisation with a learning rate that is inversely proportional to the width of the NN, which has been shown to be the largest scale for stable learning rates in wide NNs \citep{karakida2019universal, park2019effect, sohl2020infinite}.

\citet{lee2019wide} built on the results of \citet{jacot2018neural}, and studied the \textit{linearised regime} of an NN. Specifically, if we denote as $f_t(\vx)=f(\vx, \vtheta_t)$ the network function at time $t$, we can define the first order Taylor expansion of the network function around randomly initialised parameters $\vtheta_0$ to be:
\begin{align} \label{eqn:linearised}
    f_t^{\text{lin}}(\vx) = f_0(\vx) + \nabla_{\vtheta} f(\vx, \vtheta_0) \Delta\vtheta_t
\end{align}

where $\Delta\vtheta_t = \vtheta_t - \vtheta_0$ and $f_0 = f(\cdot, \vtheta_0)$ is the randomly initialised NN function.

For notational clarity, whenever we evaluate a function at an arbitrary input set $\mathcal{X}'$ instead of a single point  $\vx'$, we suppose the function is vectorised. For example, $f_t(\mathcal{X})\in\mathbb{R}^{NC}$ denotes the concatenated NN outputs on training set $\mathcal{X}$, whereas $\nabla_{\vtheta}f_t(\mathcal{X}) = \nabla_{\vtheta}f(\mathcal{X}, \vtheta_t)\in \mathbb{R}^{NC\times p}$. In the interest of space, we will also sometimes use subscripts to signify kernel inputs, so for instance $\Thetaxy{\vx'}{\X}^{\mathstrut}=\Theta(\vx', \mathcal{X})\in\mathbb{R}^{C\times NC}$ and $\ThetaXX^{\mathstrut}=\Theta(\X, \X)\in\mathbb{R}^{NC\times NC}$  throughout this work.

The results of \citet{lee2019wide} showed that in the infinite width limit, with NTK parameterisation and gradient flow under squared error loss,  $f_t^{\text{lin}}(\vx)$ and $f_t(\vx)$ are equal for any $t\geq0$, for a shared random initialisation $\vtheta_0$. In particular, for the linearised network it %
can be shown, that as $t{\rightarrow} \infty$:
\begin{align} f_{\infty}^{\text{lin}}(\vx) = f_0(\vx) - \hat \Theta_0(\vx, \mathcal X) \hat \Theta_0(\mathcal{X}, \mathcal{X})^{-1}(f_0(\mathcal X) - \mathcal Y)
\end{align}
and thus as the hidden layer widths converge to infinity we have that:
\begin{align} \label{align:limit}
f_{\infty}^{\text{lin}}(\vx) = f_{\infty}(\vx) = f_0(\vx) - \Theta(\vx, \mathcal X) \Theta(\mathcal{X},\mathcal{X})^{-1}(f_0(\mathcal X) - \mathcal Y)
\end{align}

We can replace $\Theta(\mathcal X, \mathcal X)^{-1}$ with the generalised inverse when invertibility is an issue. However, this will not be a main concern of this work, as our methods will add regularisation that corresponds to modelling observation/output noise, which both ensures invertibility and alleviates any potential convergence issues due to fast decay of the NTK eigenspectrum \citep{ronen2019convergence}.

From Eq.~(\ref{align:limit}) we see that, conditional on the training data $\{\mathcal{X}, \mathcal{Y}\}$, we can decompose $f_{\infty}$ into $f_{\infty}(\vx)=\mu(\vx) \plus \gamma(\vx)$ where  $\mu(\vx) =\Theta(\vx, \mathcal X) \Theta(\mathcal{X}, \mathcal{X})^{-1} \mathcal Y $
is a deterministic mean and $\gamma(\vx) = f_0(\vx) \minus \Theta(\vx, \mathcal X) \Theta(\mathcal{X}, \mathcal{X})^{-1}f_0(\mathcal X)$ captures predictive uncertainty, due to the randomness of $f_0$.
Now, if we suppose that, at initialisation, $f_0\overset{d}{\sim} \mathcal{GP}(0, k)$ for an arbitrary kernel $k:\mathbb{R}^d\times\mathbb{R}^d\rightarrow \mathbb{R}^{C\times C}$, then we have
$f_{\infty}(\cdot)\overset{d}{\sim} \mathcal{GP}(\mu(\vx), \Sigma(\vx, \vx'))$ for two inputs $\vx, \vx'$,
where:\footnote{Throughout this work, the notation ``$\plus h.c.$'' means ``plus the Hermitian conjugate'', like \citet{lee2019wide}. \\For example:  $\Theta^{\mathstrut}_{\vx \mathcal{X}}\ThetaXX^{-1}k^{\mathstrut}_{\mathcal{X} \vx'} + h.c. = \Theta^{\mathstrut}_{\vx \mathcal{X}}\ThetaXX^{-1}k^{\mathstrut}_{\mathcal{X} \vx'} + \Theta^{\mathstrut}_{\vx' \mathcal{X}}\ThetaXX^{-1}k^{\mathstrut}_{\mathcal{X} \vx}$}
\begin{align}\label{eqn:covariance}
\Sigma(\vx, \vx') = k^{\mathstrut}_{\vx \vx'} + \Theta_{\vx\mathcal X}^{\mathstrut}\ThetaXX^{-1}k_{\X \X}^{\mathstrut}\ThetaXX^{-1}\Theta^{\mathstrut}_{\mathcal X \vx} - \big(\Theta^{\mathstrut}_{\vx \mathcal{X}}\ThetaXX^{-1}k^{\mathstrut}_{\mathcal{X} \vx'} + h.c.\big)
\end{align}
 For a generic kernel $k$, \citet{lee2019wide} observed that this limiting distribution for $f_{\infty}$ does not have a posterior GP interpretation unless $k$ and $\Theta$ are multiples of each other, .

As mentioned in Section \ref{sec:intro}, previous work \citep{neal1996priors, lee2018deep, alexander2018matthews, garriga-alonso2018deep, novak2019bayesian, yang2019tensori, hron2020infinite} has shown that there is a distinct but closely related kernel $\K$, known as the \textit{Neural Network Gaussian Process} (NNGP) kernel, such that $f_0\overset{d}{\rightarrow}\mathcal{GP}(0, \K)$ at initialisation in the infinite width limit and $\K\neq \Theta$. Thus Eq.~(\ref{eqn:covariance}) with $k{=}\K$ tells us that, for wide NNs under squared error loss, there is no Bayesian posterior interpretation to a trained NN, nor is there an interpretation to a trained deep ensemble as a Bayesian posterior predictive approximation.
\section{Proposed modification to obtain posterior samples in infinite width limit}\label{section:correction}
 \citet{lee2019wide} noted that one way to obtain a posterior interpretation to $f_{\infty}$ is by randomly initialising $f_0$ but only training the parameters in the final linear readout layer, as the contribution to the NTK $\Theta$ from the parameters in final hidden layer is exactly the NNGP kernel $\K$.\footnote{Up to a multiple of last layer width in standard parameterisation.} $f_{\infty}$ is then a sample from the GP posterior with prior kernel NNGP, $\mathcal{K}$, and noiseless observations in the infinite width limit i.e. $f_{\infty}(\X')\overset{d}{\sim}\mathcal N(\K_{\X' \X}^{\mathstrut}\KXX^{-1}\Y, \hspace{1mm}\K_{\X' \X'}^{\mathstrut} \minus  \K^{\mathstrut}_{\X' \X}\KXX^{-1}\K_{\X \X'}^{\mathstrut})$. This is an example of the ``sample-then-optimise'' procedure of \citet{matthewssample}, but, by only training the final layer this procedure limits the earlier layers of an NN solely to be random feature extractors.

We now introduce our modification to standard training that trains all layers of a finite width NN and obtains an exact posterior interpretation in the infinite width limit with NTK parameterisation and squared error loss. For notational purposes, let us suppose $\vtheta\myeq\texttt{concat}(\{\vtheta^{\leq L}, \vtheta^{L\plus1}\})$ with $\vtheta^{\leq L}\myin \mathbb{R}^{p{-}p_{L{+}1}}$ denoting $L$ hidden layers, and $\vtheta^{L\plus 1}\myin\mathbb{R}^{p_{L\plus1}}$ denoting final readout layer $L\plus1$. Moreover, define $\Theta^{\leq L} \myeq \Theta - \mathcal{K}$
to be the p.d. kernel corresponding to  contributions to the NTK from all parameters before the final layer, and $\hat{\Theta}_t^{\leq L}$ to be the empirical counterpart depending on $\vtheta_t$. To motivate our modification, we reinterpret $f_t^{\text{lin}}$ in \cref{eqn:linearised} by splitting terms related to $\K$ and $\Theta^{\leq L}$:
\begin{align} \label{eqn:new_linearised}
    f_t^{\text{lin}}(\vx) = \underbrace{f_0(\vx) + \nabla_{\vtheta^{L+1}} f(\vx, \vtheta_0) \Delta\vtheta^{L+1}_t}_{\K} +  \underbrace{\mathbf{0}_C + \nabla_{\vtheta^{\leq{L}}} f(\vx, \vtheta_0) \Delta\vtheta^{\leq L}_t}_{\Theta - \K}
\end{align}
where $\mathbf{0}_C \myin \mathbb{R}^C$ is the zero vector. As seen in \cref{eqn:new_linearised}, the distribution of  $f_0^{\text{lin}}(\vx){=}f_0(\vx)$ lacks extra variance, $\Theta^{\leq L}(\vx,\vx)$, that accounts for contributions to the NTK $\Theta$ from all parameters $\vtheta^{\leq L}$ before the final layer. This is precisely why no Bayesian intepretation exists for a standard trained wide NN, as in \cref{eqn:covariance} with $k\myeq\K$. The motivation behind our modification is now very simple: we propose to manually add in this missing variance. Our modified NNs, $\tilde f(\cdot, \vtheta)$, will then have trained distribution:
\begin{align}\label{eqn:ntkgp_dist}
    \tilde f_{\infty}(\X')\overset{d}{\sim}\mathcal N(\Theta_{\X' \X}^{\mathstrut}\ThetaXX^{-1}\Y, \hspace{1mm}\Theta_{\X' \X'}^{\mathstrut} \minus  \Theta^{\mathstrut}_{\X' \X}\ThetaXX^{-1}\Theta_{\X \X'}^{\mathstrut})
\end{align}
on a test set $\X'$, in the infinite width limit. Note that Eq.~(\ref{eqn:ntkgp_dist}) is the GP posterior using prior kernel $\Theta$ and noiseless observations $\tilde f_{\infty}(\mathcal{X}){=}\mathcal{Y}$, which we will refer to as the \textit{NTKGP} posterior predictive.
We construct $\tilde f$ by sampling a random and untrainable function $\delta(\cdot)$ that is added to the standard forward pass $f(\cdot, \vtheta_t)$, defining an augmented forward pass:
\begin{align}\label{eqn:ftilde}
    \tilde f(\cdot, \vtheta_t) = f(\cdot, \vtheta_t) + \delta(\cdot)
\end{align}
\vspace{-.5em}
Given a parameter initialisation scheme $\texttt{init}(\cdot)$ and initial parameters $\vtheta_0\mysim\texttt{init}(\cdot)$, our chosen formulation for $\delta(\cdot)$ is as follows:
 1) sample  $\tilde{\vtheta}\mysim\texttt{init}(\cdot)$ independently of $\vtheta_0$; 2) denote $\tilde{\vtheta}\myeq\texttt{concat}(\{\tilde{\vtheta}^{\leq L}, \tilde{\vtheta}^{L+1}\})$; and 3) define  ${\vtheta}^* \myeq \texttt{concat}(\{\tilde{\vtheta}^{\leq L}, \mathbf{0}_{p_{L+1}}\})$.
In words, we set the parameters in the final layer of an independently sampled $\tilde{\vtheta}$ to zero to obtain $\vtheta^*$. Now, we define:
\begin{align}\label{eqn:jvp1}
    \delta(\vx) = \nabla_{\vtheta}f(\vx, \vtheta_0) \vtheta^*
\end{align}

There are a few important details to note about $\delta(\cdot)$ as defined in Eq.~(\ref{eqn:jvp1}). First, $\delta(\cdot)$ has the same distribution in both NTK and standard parameterisations,\footnote{In this work $\Theta$ always denotes the NTK under NTK parameterisation. It is also possible to model $\Theta$ to be the scaled NTK under standard parameterisation (which depends on layer widths) as in \citet{sohl2020infinite} with minor reweightings to both $\delta(\cdot)$ and, when modelling observation noise, the $L^2$-regularisation described in \cref{sec:reg}.} and also $\delta(\cdot)\given\vtheta_0 \overset{d}{\sim}\mathcal{GP}\big(0, \hat{\Theta}_0^{\leq L}\big)$ in the NTK parameterisation.\footnote{With NTK parameterisation, it is easy to see that $\delta(\cdot)\given\vtheta_0 \overset{d}{\sim}\mathcal{GP}\big(0, \hat{\Theta}_0^{\leq L}\big)$, because $\tilde{\vtheta}^{\leq L }\mysim\mathcal{N}(0,I_{p-p_{L+1}})$. To extend this to standard parameterisation, note that Eq.~(\ref{eqn:jvp1}) is just the first order term in the Taylor expansion of $f(\vx, \vtheta_0 + \vtheta^*)$, which has a parameterisation agnostic distribution, about $\vtheta_0$.} Moreover, Eq.~(\ref{eqn:jvp1}) can be viewed as a single Jacobian-vector product (JVP), which packages that offer forward-mode autodifferentiation (AD), such as JAX \citep{jax2018github}, are efficient at computing for finite NNs. It is worth noting that our modification adds only negligible computational and memory requirements on top of standard deep ensembles \citep{balajideepensembles}: a more nuanced comparison can be found in Appendix \ref{app:costs}.  Alternative constructions of $\tilde f$ are presented in Appendix \ref{sec:alternative_constructions}.

To ascertain whether a trained $\tilde f_{\infty}$ constructed via Eqs. (\ref{eqn:ftilde}, \ref{eqn:jvp1}) returns a sample from the GP posterior Eq.~(\ref{eqn:ntkgp_dist}) for wide NNs, the following proposition, which we prove in Appendix~\ref{sec:proof_prop_inf_width_dist}, will be useful:
\begin{restatable}{prop}{infwidthdistprop}
\label{prop:inf_width_dist}
    $ \delta(\cdot) \overset{d}{\rightarrow} \mathcal{GP}(0, \Theta^{\leq L})$ and is independent of $f_0(\cdot)$ in the infinite width limit. Thus,  $\tilde f_0(\cdot) = f_0(\cdot) + \delta(\cdot)\overset{d}{\rightarrow}\mathcal {GP}(0, \Theta)$.
\end{restatable}
Using Proposition~\ref{prop:inf_width_dist}, we now consider the linearisation of $\tilde f_t(\cdot)$, noting that $\nabla_{\vtheta}\tilde f_0(\cdot)=\nabla_{\vtheta} f_0(\cdot)$:
\begin{align} \label{eqn:ntkgp_linearised}
    \tilde f_t^{\text{lin}}(\vx) = \underbrace{f_0(\vx) + \nabla_{\vtheta^{L+1}} f(\vx, \vtheta_0) \Delta\vtheta^{L+1}_t}_{\K} +  \underbrace{\delta(\vx) + \nabla_{\vtheta^{\leq{L}}} f(\vx, \vtheta_0) \Delta\vtheta^{\leq L}_t}_{\Theta - \K}
\end{align}
The fact that $\nabla_{\vtheta}\tilde f_t^{\text{lin}}(\cdot)\myeq\nabla_{\vtheta} f_0(\cdot)$ is crucial in \cref{eqn:ntkgp_linearised}, as this initial Jacobian is the feature map of the linearised NN regime from \citet{lee2019wide}. As per Proposition \ref{prop:inf_width_dist} and \cref{eqn:ntkgp_linearised}, we see that $\delta(\vx)$ adds the extra randomness missing from $f_0^{\text{lin}}(\vx)$ in  \cref{eqn:new_linearised}, and reinitialises $\tilde{f}_0$ as a sample from $\mathcal{GP}(0, \mathcal{K})$ to $\mathcal{GP}(0, \Theta)$ for wide NNs. This means we can set $k\myeq\Theta$ in Eq.~(\ref{eqn:covariance}) and deduce:
\begin{corollary}\label{corr:true_ntkgp_post}
$    \tilde f_{\infty}(\X')\overset{d}{\sim}\mathcal N(\Theta_{\X' \X}^{\mathstrut}\ThetaXX^{-1}\Y, \hspace{1mm}\Theta_{\X' \X'}^{\mathstrut} \minus  \Theta^{\mathstrut}_{\X' \X}\ThetaXX^{-1}\Theta_{\X \X'}^{\mathstrut})$, and hence a trained $\tilde f_{\infty}$ returns a sample from the posterior NTKGP in the infinite width limit.
\end{corollary}
To summarise: we define our new NN forward pass to give $\tilde{f}_t(\vx) = f_t(\vx) + \delta(\vx)$ for standard forward pass $f_t(\vx)$, and an untrainable $\delta(\vx)$ defined as in Eq.~(\ref{eqn:jvp1}). As given by Corollary~\ref{corr:true_ntkgp_post}, independently trained baselearners $\tilde f_{\infty}$ can then be ensembled to approximate the NTKGP posterior predictive.\\ We will call $\tilde{f}_{\infty}$ trained in this section an \NTKGP~baselearner, regardless of parameterisation or width. We are aware that the name NTK-GP has been used previously to refer to Eq.~(\ref{eqn:covariance}) with NNGP kernel $\K$, which is what standard training under squared error with a wide NN yields. However, we believe GPs in machine learning are synonymous with probabilistic inference \citep{rasmussen2003gaussian}, which Eq.~(\ref{eqn:covariance}) has no connection to in general, so we feel the name \NTKGP~is more appropriate for our methods.

\subsection{Modelling observation noise}\label{section:observation_noise}
So far, we have used %
squared loss
$\ell(y, \tilde{f}(\vx, \vtheta)) = \frac{1}{2\sigma^2}(y - \tilde{f}(\vx, \vtheta))^2$
for $\sigma^2{=}1$, and seen how our \NTKGP~training scheme for $\tilde{f}$ gives a Bayesian interpretation to trained networks when we assume noiseless observations. Lemma 3 of  \citet{osband2018randomized} shows us how to draw a posterior sample for linear $\tilde{f}$ if we wish to model Gaussian observation noise $y\overset{d}{\sim} \mathcal{N}(\tilde{f}(\vx, \vtheta), \sigma^2)$ for $\sigma^2{>}0$: by adding i.i.d. noise to targets $y'_n\overset{d}{\sim}\mathcal{N}(y_n, \sigma^2)$ and regularising $\mathcal{L}(\vtheta)$ with a weighted $L^2$ term, either $\norm{\vtheta}_{\Lambda}^2$ or  $\norm{\vtheta \minus \vtheta_0}_{\Lambda}^2$, depending on if you regularise in function space or parameter space. The weighting $\Lambda$ is detailed in Appendix \ref{sec:reg}. These methods were introduced by \citet{osband2018randomized} for the application of Q-learning in deep reinforcement learning, and are known as \textit{Randomised Prior parameter} (RP-param) and \textit{Randomised Prior function} (RP-fn) respectively.
The randomised prior (RP) methods were motivated by a Bayesian linear regression approximation of the NN, but they do not take into account the difference between the NNGP and the NTK. Our \NTKGP~methods can be viewed as a way to fix this for both the parameter space or function space methods, which we will name \NTKGP-param and \NTKGP-fn respectively. Similar regularisation ideas were explored in connection to the NTK by \citet{Hu2020Simple}, when the NN function is initialised from the origin, akin to kernel ridge regression.

\subsection{Comparison of predictive distributions in infinite width}\label{sec:preddist}
Having introduced the different ensemble training methods considered in this paper: NNGP; deep ensembles; randomised prior; and NTKGP, we will now compare their predictive distributions in the infinite width limit with squared error loss.
Table \ref{table:predcomp} displays these limiting distributions,  $f_{\infty}(\cdot) \overset{d}{\sim}\mathcal{GP}(\mu, \Sigma)$, and should be viewed as an extension to Equation (16) of \citet{lee2019wide}.
\vspace{-0.5em}
\begin{table}[h]
\centering
\caption{Predictive distributions of wide ensembles for various training methods. \textit{std} denotes standard training with $f(\vx, \vtheta)$, and \textit{ours} denotes training using our additive $\delta(\vx)$ to make $\tilde f(\vx, \vtheta)$.}
\label{table:predcomp}
\resizebox{\textwidth}{!}{
\begin{tabular}{@{}ccccc@{}}
\addlinespace[.5em]
\toprule
Method & \begin{tabular}[c]{@{}c@{}}Layers\\ trained \end{tabular} & \begin{tabular}[c]{@{}c@{}}Output\\ Noise \end{tabular}  & $\mu(\vx)$ & $\Sigma(\vx, \vx')$ \\ \midrule
NNGP & Final & $\sigma^2\geq0$ & $\K_{\vx \X}^{\mathstrut}(\KXX^{\mathstrut} \plus \sigma^2I)^{-1}\Y$ & $\K_{\vx \vx'}^{\mathstrut} - \K_{\vx \X}^{\mathstrut}(\KXX^{\mathstrut} \plus \sigma^2I)^{-1}\K^{\mathstrut}_{\X \vx'}$     \\ \addlinespace[1em]
\begin{tabular}[c]{@{}c@{}}Deep\\Ensembles \end{tabular}& All (\textit{std}) & $\sigma^2=0$ & $\Theta^{\mathstrut}_{\vx \X}\ThetaXX^{-1}\Y$ & \begin{tabular}[c]{@{}c@{}}$\K_{\vx \vx'}^{\mathstrut}- \big(\Theta_{\vx \X}^{\mathstrut}\ThetaXX^{-1}\K^{\mathstrut}_{\X \vx'} \plus h.c.\big)$\\$ \Theta^{\mathstrut}_{\vx \X}\ThetaXX^{-1}\KXX^{\mathstrut}\ThetaXX^{-1}\Theta^{\mathstrut}_{\X \vx'}$ \end{tabular}    \\ \addlinespace[1em]
\begin{tabular}[c]{@{}c@{}}Randomised\\ Prior \end{tabular}& All (\textit{std}) & $\sigma^2>0$ & $\Theta_{\vx \X}^{\mathstrut}(\ThetaXX^{\mathstrut} \plus \sigma^2I)^{-1}\mathcal Y$                 &  \begin{tabular}[c]{@{}c@{}}$\K_{\vx \vx'}^{\mathstrut}- \big(\Theta^{\mathstrut}_{\vx \X}(\ThetaXX^{\mathstrut} \plus \sigma^2 I )^{-1}\K_{\X \vx'}^{\mathstrut} \plus h.c.\big)$\\$+ \Theta_{\vx \X}^{\mathstrut}(\ThetaXX^{\mathstrut} \plus \sigma^2 I )^{-1}(\KXX^{\mathstrut} \plus \sigma^2I)(\ThetaXX^{\mathstrut} \plus \sigma^2 I )^{-1}\Theta_{\X \vx'}^{\mathstrut}$\end{tabular}    \\ \addlinespace[1em]
NTKGP & All (\textit{ours}) & $\sigma^2\geq0$ & $\Theta_{\vx \X}^{\mathstrut}(\Theta_{\X \X}^{\mathstrut} \plus \sigma^2I)^{-1}\mathcal Y$                 & $\Theta_{\vx \vx'}^{\mathstrut} - \Theta_{\vx \X}^{\mathstrut}(\Theta_{\X \X}^{\mathstrut} \plus \sigma^2I)^{-1}\Theta_{\X \vx'}^{\mathstrut}$     \\ \bottomrule
\end{tabular}}

\end{table}

In order to parse Table \ref{table:predcomp}, let us denote $\mu_{\text{NNGP}}, \mu_{\text{DE}}, \mu_{\text{RP}}$, $\mu_{\text{NTKGP}}$ and $\Sigma_{\text{NNGP}}, \Sigma_{\text{DE}}, \Sigma_{\text{RP}}$, $\Sigma_{\text{NTKGP}}$ to be the entries in the $\mu(\vx)$ and $\Sigma(\vx, \vx')$ columns of Table \ref{table:predcomp} respectively, read from top to bottom.
We see that $\mu_{\text{DE}}(\vx)=\mu_{\text{NTKGP}}(\vx)$ if $\sigma^2{=}0$, and $\mu_{\text{RP}}(\vx)=\mu_{\text{NTKGP}}(\vx)$ if $\sigma^2{>}0$. In words: the predictive mean of a trained ensemble is the same when training all layers, both with standard training and our \NTKGP~training.
This holds because both $f_0$ and $\tilde{f}_0$ are zero mean. It is also possible to compare the predictive covariances as the following proposition, proven in \cref{section:prop_proof_trained_cov}, shows:

\begin{restatable}{prop}{trainedcovprop}\label{prop:trained_cov}
     For $\sigma^2{=}0$, $\Sigma_{\text{NTKGP}} \succeq \Sigma_{\text{DE}} \succeq \Sigma_{\text{NNGP}}$. Similarly, for $\sigma^2{>}0$, $\Sigma_{\text{NTKGP}} \succeq \Sigma_{\text{RP}} \succeq \Sigma_{\text{NNGP}}$.
\end{restatable}

Here, when we write $k_1 \succeq k_2$ for p.d. kernels $k_1, k_2$, we mean that $k_1 \minus k_2$ is also a p.d. kernel.\\ One consequence of Proposition \ref{prop:trained_cov} is that the predictive distribution of an ensemble of NNs trained via our \NTKGP~methods is always more conservative than a standard deep ensemble, in the linearised NN regime, when the ensemble size $K{\rightarrow}\infty$. It is not possible to say in general when this will be beneficial, because in practice our models will always be misspecified. However, Proposition \ref{prop:trained_cov} suggests that in situations where we suspect standard deep ensembles might be overconfident, such as in situations where we expect some dataset shift at test time, our methods should hold an advantage.

\subsection{Modelling heteroscedasticity}
Following \citet{balajideepensembles}, if we wish to model heteroscedasticity in a univariate regression setting such that each training point, $(\vx_n, y_n)$, has an individual observation noise $\sigma^2(\vx_n)$ then we use the heteroscedastic Gaussian NLL loss (up to additive constant):
\begin{align}\label{eqn:heteroscedastic}
    \ell(y'_n, \tilde{f}(\vx_n, \vtheta))=\frac{(y'_n-\tilde{f}(\vx_n, \vtheta))^2}{2\sigma^2(\vx_n)} + \frac{\log\sigma^2(\vx_n)}{2}
\end{align}
where $y'_n \myeq y_n \plus \sigma(\vx_n) \epsilon_n$ and $\epsilon_n\overset{\text{i.i.d.}}{\sim}\mathcal{N}(0,1)$.
    It is easy to see that for fixed $\sigma^2(\vx_n)$, our \NTKGP~trained baselearners will still have a Bayesian interpretation: $\mathcal{Y}'\leftarrow \Sigma^{-\frac{1}{2}}\mathcal{Y}'$ and $\tilde{f}(\mathcal{X}, \vtheta)\leftarrow\Sigma^{-\frac{1}{2}}\tilde{f}(\mathcal{X}, \vtheta)$ returns us to the homoscedastic case, where $\Sigma\myeq\text{diag}(\sigma^2(\mathcal{X}))\myin \mathbb{R}^{N\times N}$.
We will follow \citet{balajideepensembles} and parameterise $\sigma^2(\vx)\myeq\sigma^2_{\vtheta}(\vx)$ by an extra output head of the NN, that is trainable alongside the mean function $\mu_{\vtheta}(\vx)$ when modelling heteroscedasticity.\footnote{We use the \textit{sigmoid} function, instead of \textit{softplus} \citep{balajideepensembles}, to enforce positivity on $\sigma_{\vtheta}^2(\cdot)$, because our data will be standardised.}
\subsection{NTKGP Ensemble Algorithms}\label{section:algorithms}
We now proceed to train an ensemble of $K$ NTKGP baselearners. Like previous work \citep{balajideepensembles, osband2018randomized}, we independently initialise baselearners, and also use a fixed, independently sampled training set noise $\bm{\epsilon}_k\myin \mathbb{R}^{NC}$ if modelling output noise. These implementation details are all designed to encourage diversity among baselearners, with the goal of approximating the NTKGP posterior predictive for our Bayesian deep ensembles. \cref{app:aggregation} details how to aggregate predictions from trained baselearners. In Algorithm \ref{alg:ntkgp_param}, we outline our NTKGP-param method: \texttt{data\_noise} adds observation noise to targets; \texttt{concat} denotes a concatenation operation; and $\texttt{init}(\cdot)$ will be standard parameterisation initialisation in the JAX library Neural Tangents \citep{neuraltangents2020} unless stated otherwise. As discussed by \citet{pearce2018uncertainty}, there is a choice between ``anchoring''/regularising parameters towards their initialisation or an independently sampled parameter set when modelling observation noise. We anchor at initialisation as the linearised NN regime only holds local to parameter initialisation \citep{lee2019wide}, and also this reduces the memory cost of sampling parameters sets. Appendix \ref{app:algorithms} details our NTKGP-fn method.
\vspace{-.3em}
\begin{algorithm}[h]
\caption{\NTKGP-param ensemble}
\begin{algorithmic}\label{alg:ntkgp_param}
\REQUIRE Data $\train=\{\X,\Y\}$, loss function $\mathcal{L}$, NN model $f_{\vtheta}:\mathcal{X}\rightarrow \mathcal{Y}$, Ensemble size $K\in\mathbb{N}$, noise procedure: \texttt{data\_noise}, NN parameter initialisation scheme: \texttt{init}($\cdot$)  \\
\FOR{$k=1,\ldots,K$}
    \STATE Form $\{\mathcal X_k, \mathcal Y_k\} = \texttt{data\_noise}(\mathcal{D})$\;
    \STATE Initialise $\vtheta_k\overset{d}{\sim} \texttt{init}(\cdot)$\;
    \STATE Initialise $\tilde{\vtheta}_{k}\overset{d}{\sim} \texttt{init}(\cdot)$ and denote $\tilde{\vtheta}_k=\texttt{concat}(\{\tilde{\vtheta}_k^{\leq L}, \tilde{\vtheta}_k^{L+1}\})$
    \STATE Set ${\vtheta}_k^* = \texttt{concat}(\{\tilde{\vtheta}_k^{\leq L}, \mathbf{0}_{p_{L\plus1}}\})$
    \STATE Define $\delta(\vx) = \nabla_{\vtheta}f(\vx, \vtheta_k)\vtheta^{*}_{k}$ \;
    \STATE Define $\tilde{f}_k(\vx, \vtheta_t) = f(\vx, \vtheta_t) + \delta(\vx)$ and set $\vtheta_0=\vtheta_k$\;
    \STATE Optimise $\mathcal{L}(\tilde{f}_k(\mathcal{X}_k, \vtheta_t), \mathcal{Y}_k) + \frac{1}{2}\norm{\vtheta_t - \vtheta_k}_{\Lambda}^2$ for $\vtheta_t$ to obtain $\hat\vtheta_k$
\ENDFOR
\RETURN ensemble $\{\tilde{f}_k(\cdot,\hat \vtheta_k)\}_{k=1}^K$
\end{algorithmic}
\end{algorithm}
\vspace{-1.em}
\subsection{Classification methodology}\label{section:classification}

For classification, we follow
recent works \cite{lee2019wide, arora2019exact, shankar2020neural} which treat classification as a regression task with one-hot regression targets. In order to obtain probabilistic predictions, we temperature scale our trained ensemble predictions with cross-entropy loss on a held-out validation set, noting that \citet{fong_biometrika} established a connection between marginal likelihood maximisation and cross-validation.

Because $\delta(\cdot)$ is untrainable in our NTKGP methods, it is important to match the scale of the NTK $\Theta$ to the scale of the one-hot targets in multi-class classification settings. One can do this either by introducing a scaling factor $\kappa>0$ such that we scale either: 1) $\tilde{f}\leftarrow \frac{1}{\kappa}\tilde{f}$ so that $\Theta\leftarrow \frac{1}{\kappa^2}\Theta$, or 2) $e_c\leftarrow \kappa e_c$ where $e_c\in\mathbb{R}^C$ is the one-hot vector denoting class $c\leq C$. We choose option 2) for our implementation, tuning $\kappa$ on a small set of values chosen to match the second moments of the randomly initialised baselearners, in logit space, of each ensemble method on the training set. We found $\kappa$ to be an important hyperparameter that can determine a trade-off between in-distribution and out-of-distribution performance: see Appendix \ref{app:scaling} for further details.

\section{Experiments}\label{section:experiments}
 Due to limited space, \cref{app:experimental_details} will contain all experimental details not discussed in this section.
\paragraph{Toy 1D regression task}
We begin with a toy 1D example $y = x\text{sin}(x) + \epsilon$, using homoscedastic $\epsilon\overset{d}{\sim}\mathcal{N}(0, 0.1^2)$. We use a training set of $20$ points  partitioned into two clusters, in order to detail uncertainty on out-of-distribution test data. For each ensemble method, we use MLP baselearners with two hidden layers of width 512, and erf activation.
The choice of erf activation means that both the NTK $\Theta$ and NNGP kernel $\K$ are analytically available \citep{lee2019wide, williams1997computing}. We compare ensemble methods to the analytic GP posterior using either $\Theta$ or $\K$ as prior covariance function using the Neural Tangents library \citep{neuraltangents2020}.

Figure \ref{fig:1d} compares the analytic NTKGP posterior predictive with the analytic NNGP posterior predictive, as well as three different ensemble methods: deep ensembles, RP-param and NTKGP-param. We plot 95\% predictive confidence intervals, treating ensembles as one Gaussian predictive distribution with matched moments like \citet{balajideepensembles}. As expected, both NTKGP-param and RP-param ensembles have similar predictive means to the analytic NTKGP posterior. Likewise, we see that only our NTKGP-param ensemble predictive variances match the analytic NTKGP posterior. As foreseen in Proposition \ref{prop:trained_cov}, the analytic NNGP posterior and other ensemble methods make more confident predictions than the NTKGP posterior, which in this example results in overconfidence on out-of-distribution data.\footnote{Code for this experiment is available at: \href{https://github.com/bobby-he/bayesian-ntk}{\texttt{https://github.com/bobby-he/bayesian-ntk}}.}

\begin{figure}[h]
  \centering
      \includegraphics[width=0.95\textwidth]{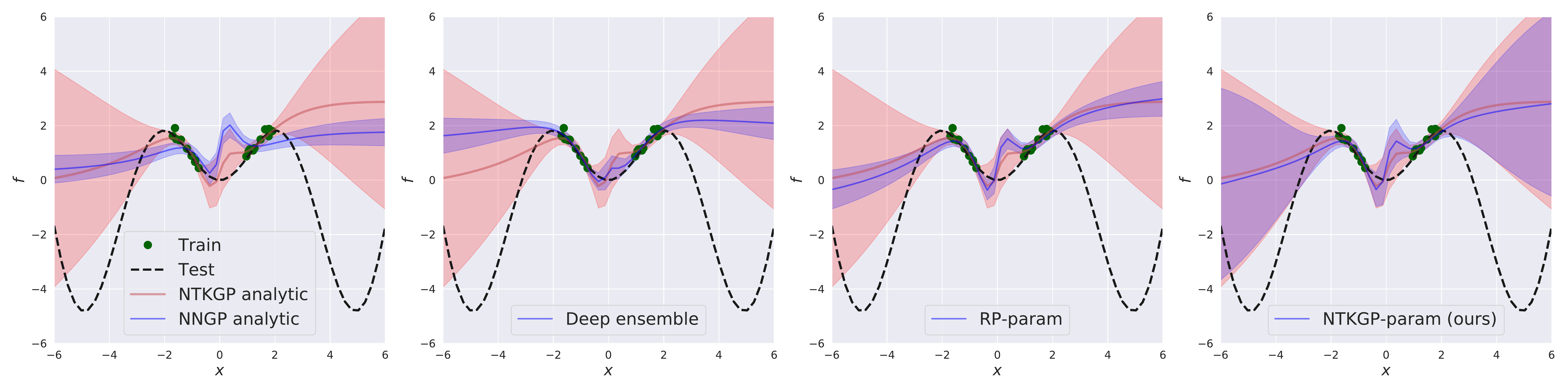}
      \vspace{-1.em}
  \caption{All subplots plot the analytic \NTKGP~posterior (in red). From left to right, (in blue): analytic NNGP posterior; deep ensembles; RP-param; and NTKGP-param (ours). For each method we plot the mean prediction and 95\% predictive confidence interval. Green points denote the training data, and the black dotted line is the true test function $y=x\text{sin}(x)$.}
  \label{fig:1d}
\end{figure}
\vspace{-1.em}
\paragraph{Flight Delays}
We now compare different ensemble methods on a large scale regression problem using the Flight Delays dataset \citep{hensman2013gaussian}, which is known to contain dataset shift. We train heteroscedastic baselearners on the first 700k data points and test on the next 100k test points at 5 different starting points: 700k, 2m (million), 3m, 4m and 5m. The dataset is ordered chronologically in date through the year 2008, so we expect the NTKGP methods to outperform standard deep ensembles for the later starting points. Figure \ref{fig:flights} (Left) confirms our hypothesis. Interestingly, there seems to be a seasonal effect between the 3m and 4m test set that results in stronger performance in the 4m test set than the 3m test set, for ensembles trained on the first 700k data points. We see that our Bayesian deep ensembles perform slightly worse than standard deep ensembles when there is little or no test data shift, but fail more gracefully as the level of dataset shift increases.

Figure \ref{fig:flights} (Right) plots confidence versus error for different ensemble methods on the combined test set of 5$\times$100k points. For each precision threshold $\tau$, we plot root-mean-squared error (RMSE) on examples where predictive precision is larger than $\tau$, indicating confidence. As we can see, our NTKGP methods incur lower error over all precision thresholds, and this contrast in performance is magnified for more confident predictions.

\begin{figure}[h]
    \centering
    \hspace*{\fill}%
    \begin{subfigure}{0.4\textwidth}
    \centering
    \includegraphics[width = \linewidth]{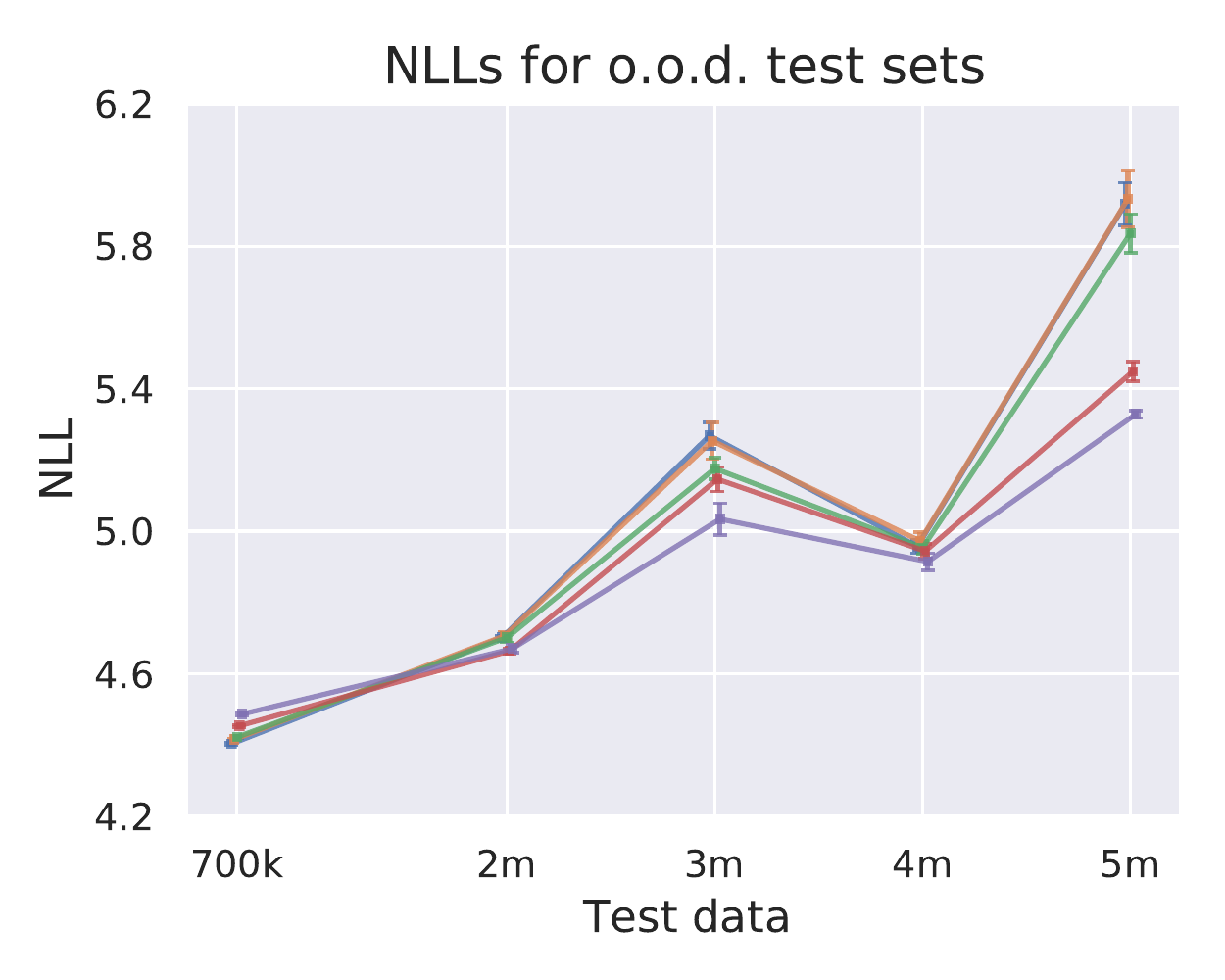}
    \end{subfigure}
    \hfill%
    \begin{subfigure}{0.4\textwidth}
    \centering
    \includegraphics[width = \linewidth]{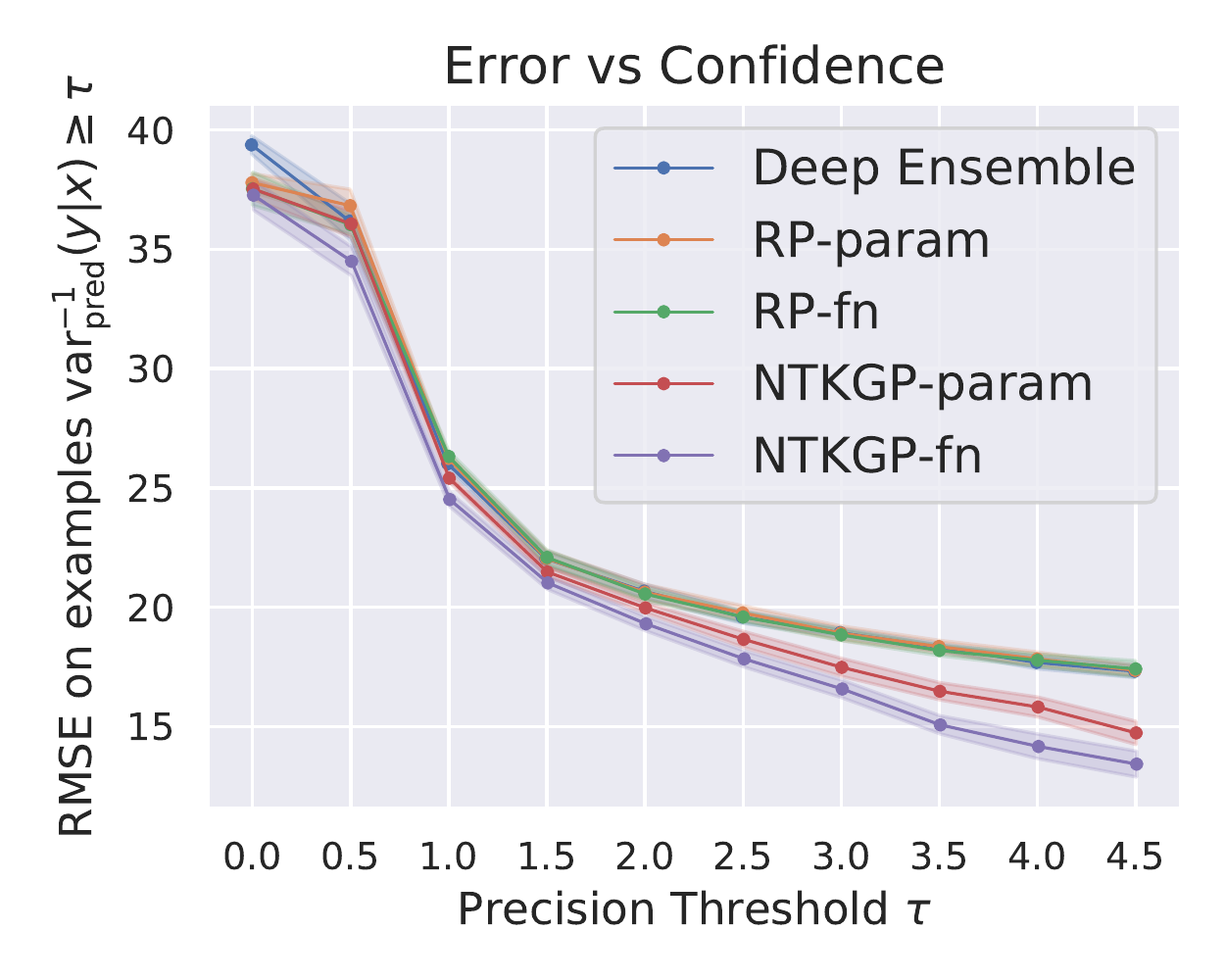}
    \end{subfigure}
    \hspace*{\fill}%
    \reducespaceafterfigure
    \caption{(Left) Flight Delays NLLs for ensemble methods trained on first 700k points of the dataset and tested on various out-of-distribution test sets, with time shift between training set and test set increasing along the $x$-axis. (Right) Error vs Confidence curves for ensembles tested on all 5$\times$100k test points combined. Both plots include 95\% CIs corresponding to 10 independent ensembles.}
    \label{fig:flights}
\end{figure}

\paragraph{MNIST vs NotMNIST} We next move onto classification experiments, comparing ensembles trained on MNIST and tested on both MNIST and NotMNIST.\footnote{Available at \url{http://yaroslavvb.blogspot.com/2011/09/notmnist-dataset.html}} Our baselearners are MLPs with 2-hidden layers, 200 hidden units per layer and ReLU activation. The weight parameter initialisation variance $\sigma_W^2$ is tuned using the validation accuracy on a small set of values around the He initialisation, $\sigma_W^2{=}2$,  \cite{He_2015_ICCV} for all classification experiments.
Figure \ref{fig:mnist} shows both in-distribution and out-of-distribution performance across different ensemble methods. In Figure \ref{fig:mnist} (left), we see that our NTKGP methods suffer from slightly worse in-distribution test performance, with around 0.2\% increased error for ensemble size $10$. However, in Figure \ref{fig:mnist} (right), we plot error versus confidence on the combined MNIST and NotMNIST test sets: for each test point $(\vx, y)$, we calculate the ensemble prediction $p(y=k|\vx)$ and define the predicted label as $\hat{y} = \text{argmax}_k p(y=k|\vx)$, with confidence $p(y=\hat y|\vx)$. Like \citet{balajideepensembles}, for each confidence threshold $0\leq\tau\leq1$, we plot the average error for all test points that are more confident than $\tau$. We count all predictions on the NotMNIST test set to be incorrect. We see in Figure \ref{fig:mnist} (right) that the NTKGP methods vastly outperform both deep ensembles and RP methods, obtaining over 15\% lower error on test points that have confidence $\tau{=}0.6$, compared to all baselines. This is because our methods correctly make much more conservative predictions on the out-of-distribution NotMNIST test set, as can be seen by \cref{fig:mnist_vs_notmnist_entropies}, which plots histograms of predictive entropies. Due to the simple MLP architecture and ReLU activation, we can compare ensemble methods to analytic NTKGP results in Figures \ref{fig:mnist} \& \ref{fig:mnist_vs_notmnist_entropies}, where we see a close match between the NTKGP ensemble methods (at larger ensemble sizes) and the analytic predictions, both on in-distribution and out-of-distribution performance.
\begin{figure}[h]
\centering
    \hspace*{\fill}%
    \begin{subfigure}{0.45\textwidth}
      \centering
          \includegraphics[width = \linewidth]{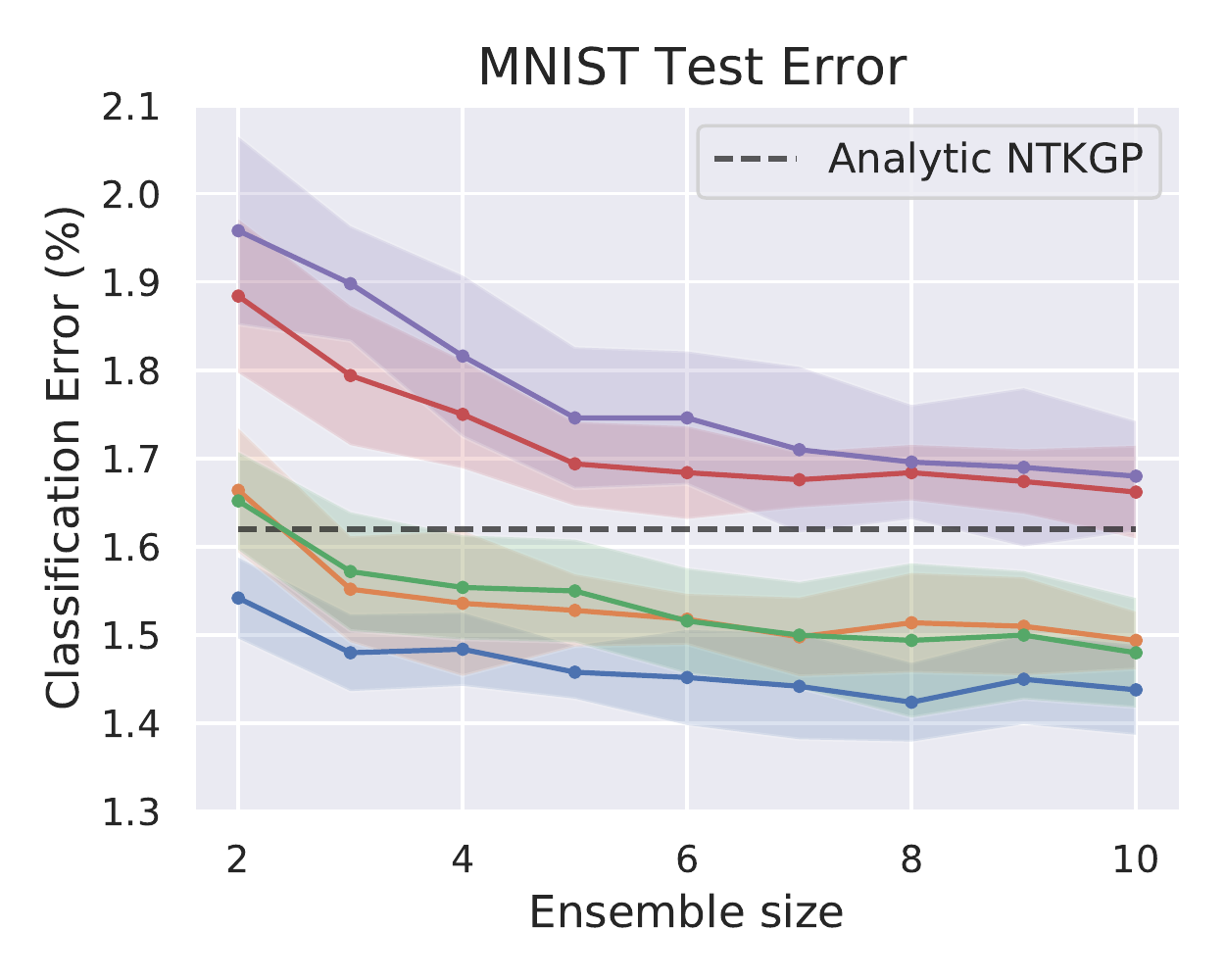}
    \end{subfigure}
    \hfill%
    \begin{subfigure}{0.45\textwidth}
      \centering
      \includegraphics[width =  \linewidth]{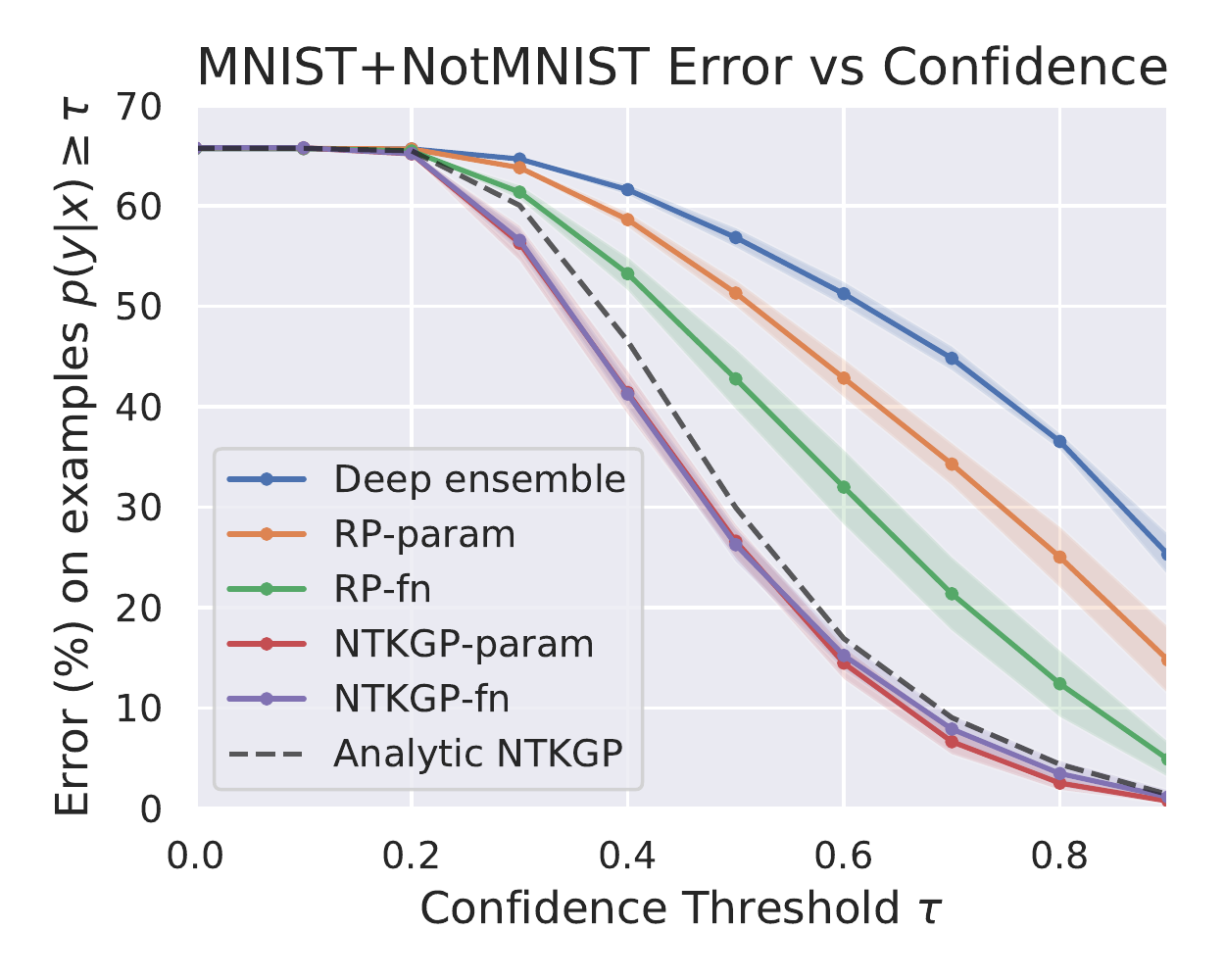}
    \end{subfigure}
    \hspace*{\fill}%
        \reducespaceafterfigure
    \caption{(Left) Classification error on MNIST test set for different ensemble sizes. (Right) Error versus Confidence plots for ensembles, of size 10, trained on MNIST and tested on both MNIST and NotMNIST. CIs correspond to 5 independent runs.}
    \label{fig:mnist}
\end{figure}

\begin{figure}[ht]
    \centering
    \begin{subfigure}[b]{0.95\textwidth}
        \includegraphics[width = \linewidth]{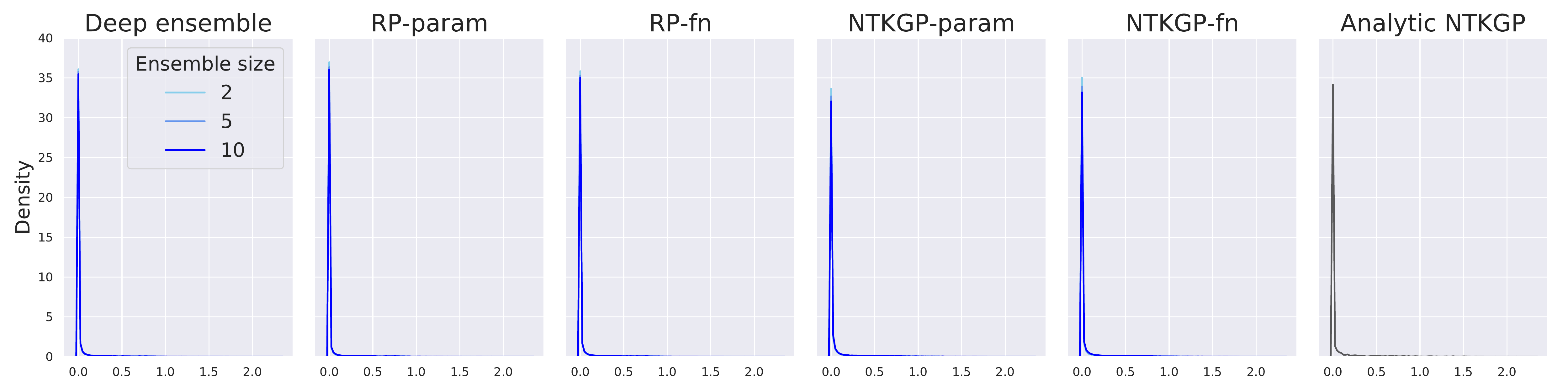}
    \end{subfigure}

    \begin{subfigure}[b]{0.95\textwidth}
        \includegraphics[width = \linewidth]{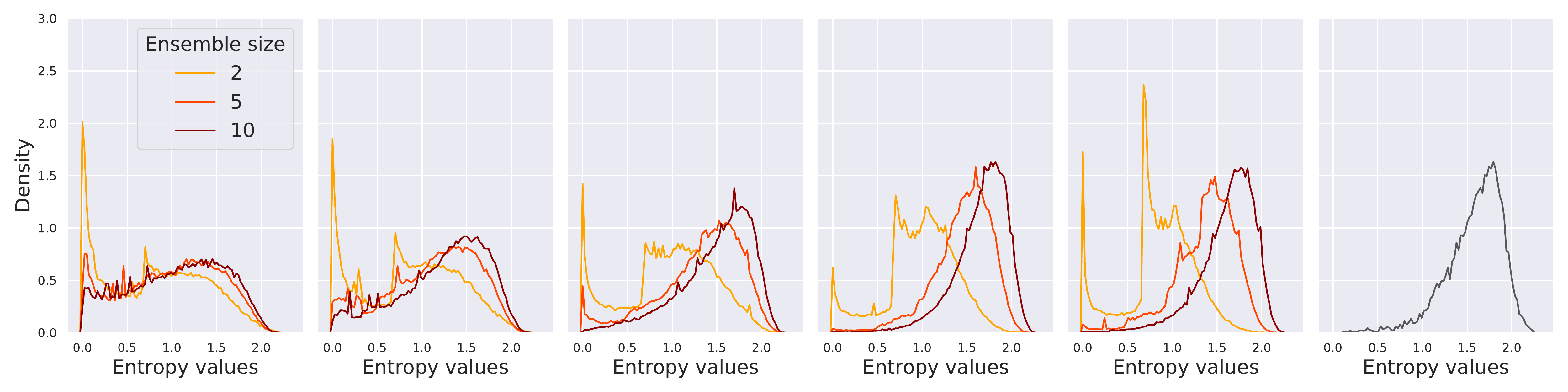}
    \end{subfigure}
    \caption{Histograms of predictive entropy on MNIST (top) and NotMNIST (bottom) test sets for different ensemble methods of different ensemble sizes, and also for Analytic NTKGP.}
    \label{fig:mnist_vs_notmnist_entropies}
\end{figure}
\paragraph{CIFAR-10 vs SVHN} Finally, we present results on a larger-scale image classification task: ensembles are trained on CIFAR-10 and tested on both CIFAR-10 and SVHN. We conduct the same setup as for the MNIST vs NotMNIST experiment, with baselearners taking the Myrtle-10 CNN architecture \cite{shankar2020neural} of channel-width 100. Figure \ref{fig:cifar10} compares in distribution and out-of-distribution performance: we see that our NTKGP methods and RP-fn perform best on in-distribution test error. Unlike on the simpler MNIST task, there is no clear difference on the corresponding error versus confidence plot, and this is also reflected in the entropy histograms, which can be found in \cref{fig:cifar10_vs_svhn_entropies} of \cref{app:experimental_details}.
    \begin{figure}[h]
\centering
    \hspace*{\fill}%
    \begin{subfigure}{0.4\textwidth}
      \centering
          \includegraphics[width = \linewidth]{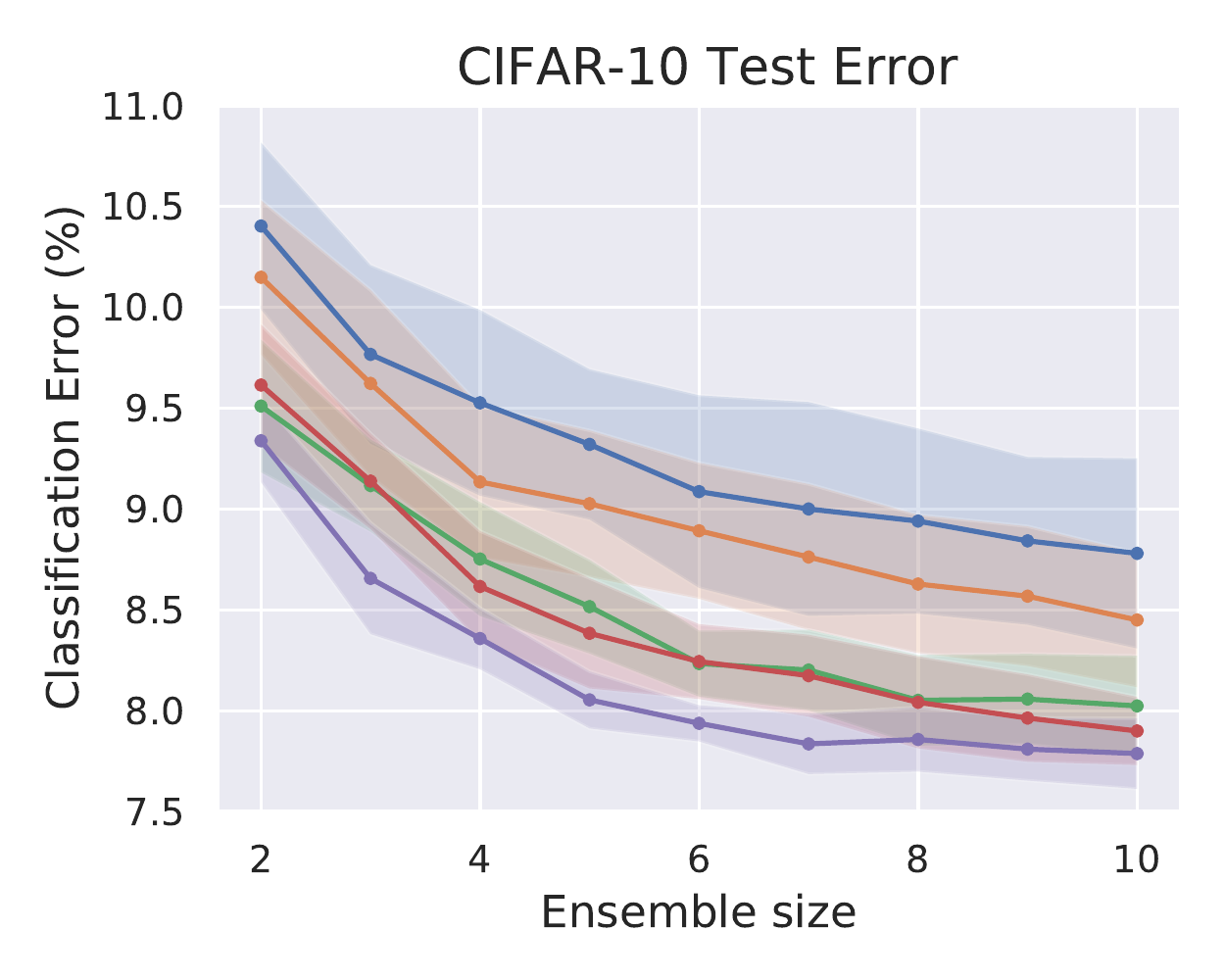}
    \end{subfigure}
    \hfill%
    \begin{subfigure}{0.4\textwidth}
      \centering
      \includegraphics[width = \linewidth]{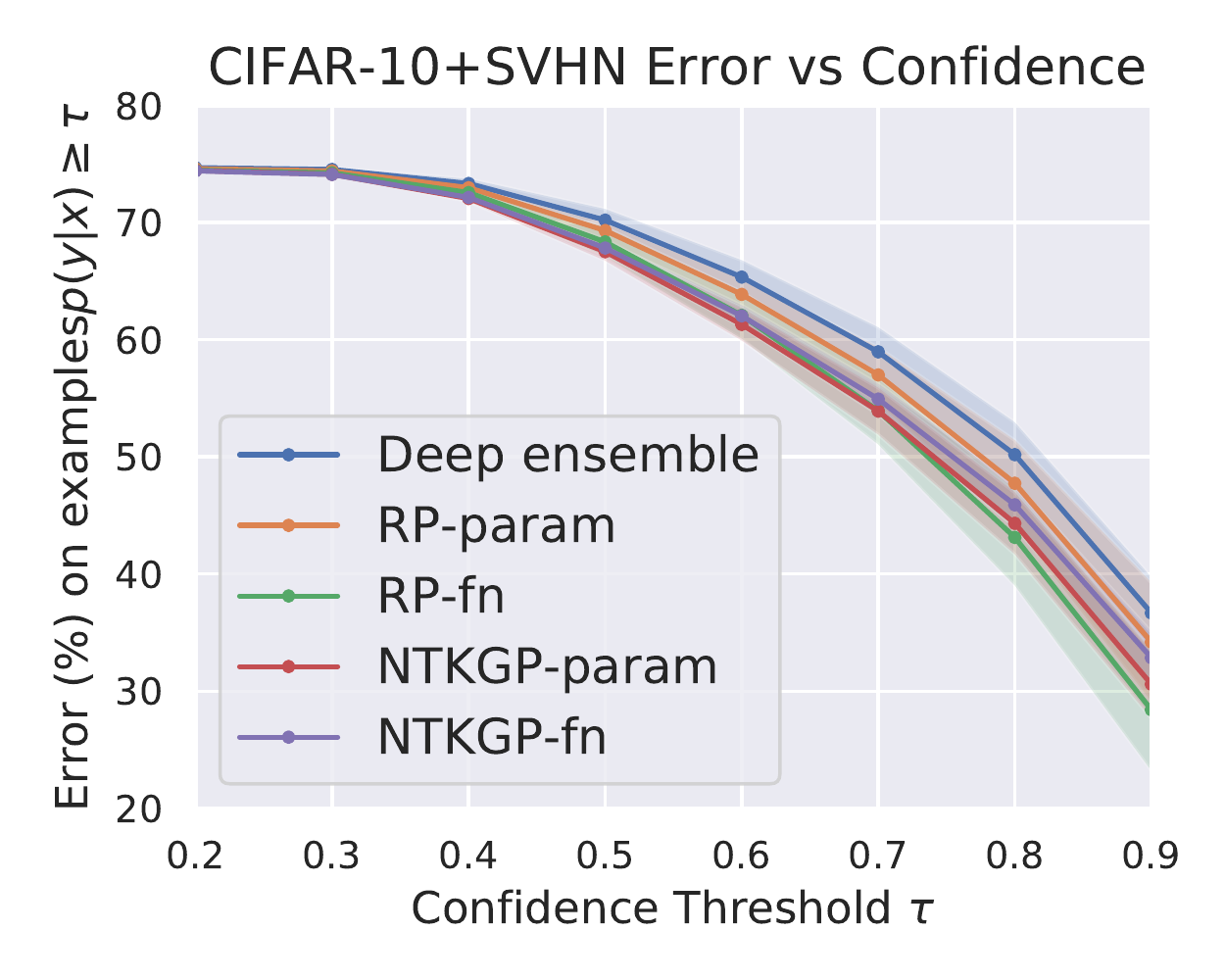}
    \end{subfigure}
    \hspace*{\fill}%
        \reducespaceafterfigure
    \caption{(Left) Classification error on CIFAR-10 test set for different ensemble sizes. (Right) Error versus Confidence plots of ensembles trained on CIFAR-10 and tested on both CIFAR-10 and SVHN. CIs correspond to 5 independent runs.}
    \label{fig:cifar10}
\end{figure}

\section{Discussion}
We built on existing work regarding the Neural Tangent Kernel (NTK), which showed that there is no posterior predictive interpretation to a standard deep ensemble in the infinite width limit. We introduced a simple modification to training that enables a GP posterior predictive interpretation for a wide ensemble, and showed empirically that our Bayesian deep ensembles emulate the analytic posterior predictive when it is available. In addition, we demonstrated that our Bayesian deep ensembles often outperform standard deep ensembles in out-of-distribution settings for both regression and classification tasks.

In terms of limitations,
our methods may perform worse than standard deep ensembles \citep{balajideepensembles} when confident predictions are not detrimental, though this can be alleviated via NTK hyperparameter tuning. Moreover,
our analyses are planted in the ``lazy learning'' regime \citep{chizat2019lazy, woodworth2020kernel}, and we have not considered finite-width corrections to the NTK during training \citep{dyer2019asymptotics, huang2019dynamics, Hanin2020Finite}. In spite of these limitations, the search for a Bayesian interpretation to deep ensembles \citep{balajideepensembles} is of particular relevance to the Bayesian deep learning community, and we believe our contributions provide useful new insights to resolving this problem by examining the limit of infinite-width.

A natural question that emerges from our work is how to tune hyperparameters of the NTK to best capture inductive biases or prior beliefs about the data. Possible lines of enquiry include: the large-depth limit \citep{hayou2019mean}, the choice of architecture \citep{zaidi2020neural}, and the choice of activation \citep{tancik2020fourier}. Finally, we would like to assess our Bayesian deep ensembles in non-supervised learning settings, such as active learning or reinforcement learning.

\clearpage

\begin{ack}
We thank
Arnaud Doucet,
Edwin Fong,
Michael Hutchinson,
Lewis Smith, Jasper Snoek, Jascha Sohl-Dickstein and
Sheheryar Zaidi, as well as the anonymous reviewers,
for helpful discussions and feedback. We also thank the JAX and Neural Tangents teams for their open-source software. BH is supported by the EPSRC and MRC through the OxWaSP CDT programme (EP/L016710/1).
\end{ack}

{%
\bibliography{references}
\bibliographystyle{unsrtnat}
}
\clearpage
\appendix
\section{Recap of standard and NTK parameterisations}\label{app:parameterisations}
For completeness, we recap the difference between standard and NTK parameterisations \& initialisations \citep{jacot2018neural, lee2019wide} for an MLP in this section.

Consider an MLP with $L$ hidden layers of  widths from $n_0{=}d$ to $n_L$ respectively, and final readout layer with $n_{L+1}=C$. For a given $\vx\in\mathbb{R}^{d}$, under the \textit{NTK} parameterisation the recurrence relation that constitutes the forward pass of the NN is then:
\begin{align}
\alpha^{(0)}(\vx, \vtheta) &= \vx \\
\tilde{\alpha}^{(l+1)}(\vx, \vtheta) &= \frac{\sigma_W}{\sqrt{n_l}} W^{(l)} \alpha^{(l)}(\vx, \vtheta) + \sigma_b b^{(l)} \\
\alpha^{(l)}(\vx, \vtheta) &= \phi (\tilde{\alpha}^{(l)}(\vx, \vtheta))
\end{align}
for $l\leq L$ where $\tilde{\alpha}^{(l)}$ and $\alpha^{(l)}$ are the preactivations and activations respectively at layer $l$, with entrywise nonlinearity $\phi(\cdot)$. In the NTK parameteriation, all parameters $W^{(l)} \in \mathbb{R}^{n_{l+1}\times n_{l}}$ and $b^{(l)} \in \mathbb{R}^{n_{l+1}}$ for all layers $l$ are initialised as i.i.d. standard normal $\mathcal N(0, 1)$. The hyperparameters $\sigma_W$ and $\sigma_b$ are known as the weight and bias variances respectively, and are hyperparameters of the infinite width limit NTK $\Theta$.

On the other hand, under \textit{standard} parameterisation, the recurrence relation of the NN is:

\begin{align}
\alpha^{(0)}(\vx, \vtheta) &= \vx \\
\tilde{\alpha}^{(l+1)}(\vx, \vtheta) &= W^{(l)} \alpha^{(l)}(\vx, \vtheta) + b^{(l)} \\
\alpha^{(l)}(\vx, \vtheta) &= \phi (\tilde{\alpha}^{(l)}(\vx, \vtheta))
\end{align}

with $W^{(l)}_{i,j}\sim \mathcal N(0, \frac{1}{n_l}\sigma^2_W)$ and $b^{(l)}_j \sim \mathcal N(0, \sigma^2_b)$ at initialisation. Commonly used initialisation schemes like LeCun \citep{lecun2012efficient} or He \citep{He_2015_ICCV} fall into this category.

Regardless of parameterisation, our notation from \cref{section:ntk,,section:correction} corresponds to $f(\vx, \vtheta)=\tilde \alpha^{(L+1)}(\vx, \vtheta)$, with $\vtheta=\{W^{(l)}, b^{(l)}\}_{l=0}^L$, $\vtheta^{\leq L} =\{W^{(l)}, b^{(l)}\}_{l=0}^{L-1}$ and $\vtheta^{L+1} =\{W^{(L)}, b^{(L)}\}$.

We see that the different parameterisations yield the same distribution for the functional output $f(\cdot, \vtheta)$ at initialisation, but give different scalings to the parameter gradients in the backward pass. \citet{sohl2020infinite} have recently explored further variants of these parameterisations.

\section{Proofs}
\subsection{Proof of Proposition~\ref{prop:inf_width_dist}}\label{sec:proof_prop_inf_width_dist}
\infwidthdistprop*
\begin{proof}
For notational ease, let us define two jointly independent GPs $g(\cdot)\overset{d}{\sim}\mathcal{GP}(0, \Theta^{\leq L})$ \& $h(\cdot)\overset{d}{\sim}\mathcal{GP}(0, \mathcal{K})$. By independence, we have $g(\cdot)+h(\cdot)\overset{d}{\sim}\mathcal{GP}(0, \Theta)$. Moreover, let $\delta_m(\cdot), f^0_m(\cdot)$ and $\vtheta^{0}_{m}$ denote $\delta(\cdot)$, $f_0(\cdot)$ and $\vtheta_0$ respectively at width parameter $m\in \mathbb{N}$. The infinite width limit thus corresponds to $m\rightarrow \infty$.

For our purposes, it will be sufficient to prove convergence of finite-dimensional marginals, $(\delta_m(\mathcal{X}), f^0_m(\mathcal{X'}))\overset{d}{\rightarrow}(g(\mathcal{X}), h(\mathcal{X'}))$ jointly, for arbitrary sets of inputs $\mathcal{X}, \mathcal{X'}$. Note that previous work \citep{lee2018deep, alexander2018matthews} has already shown that $f_m^0(\X')\overset{d}{\rightarrow}h(\X')$.

The proof that $(\delta_m(\mathcal{X}), f^0_m(\mathcal{X'}))\overset{d}{\rightarrow}(g(\mathcal{X}), h(\mathcal{X'}))$ relies on L\'evy's Convergence theorem \citep{williams_1991} and the Cram\'er-Wold device (Theorem 29.4 of \citep{Bill86}). Using these results it is sufficient to show, denoting $\varphi_X$ as the characteristic function of a random variable $X$, that:
\begin{align}
    \varphi_{Y_m}(t)\rightarrow \varphi_{Y}(t)
\end{align}
where
$Y_m = u^\top\delta_m(\mathcal{X}) + v^\top f^0_m(\mathcal{X'})$ and $Y = u^\top g(\mathcal{X}) + v^\top h(\mathcal{X'})$, for all $t\in\mathbb{R}, u\in\mathbb{R}^{|{\mathcal{X}}|C}$ and $v\in\mathbb{R}^{|\mathcal{X'}|C}$. But:
\begin{align}
    \varphi_{Y_m}(t) =& \mathbb{E}\big[\text{exp}(itY_m)\big] \\
    =& \mathbb{E}\Big[\mathbb{E}\big[\text{exp}(itY_m)\given[\big]\vtheta_{0,m}\big]\Big] \\
    =& \mathbb{E}_{\vtheta_{0,m}}\big[\text{exp}\big(-t^2u^\top \hat{\Theta}_{0,m}^{\leq L}(\mathcal{X}, \mathcal{X})u + itv^\top f_m^0(\mathcal{X'})\big)\big] \label{eqn:rem1}\\
    =& \text{exp}\big(-t^2u^\top \Theta^{\leq L}(\mathcal{X}, \mathcal{X})u \big)\mathbb{E}_{\vtheta_{0,m}}\big[\text{exp}\big(itv^\top f_m^0(\mathcal{X'})\big)\big] + r_m \label{eqn:rem2}\\
    \rightarrow& \mathbb{E}\big[\text{exp}(itY)\big]
\end{align}
where $r_m$, defined as the difference between Eqs.  (\ref{eqn:rem2}) \& (\ref{eqn:rem1}), can be shown to be $o_m(1)$ using the Bounded Convergence theorem and the empirical NTK convergence results, and by noting that proofs of NTK convergence \citep{jacot2018neural,lee2019wide,yang2019scaling, yang2020tensorii} are all done on a layer-by-layer basis.

The claim that $\tilde f_0(\cdot) = f_0(\cdot) + \delta(x)\overset{d}{\rightarrow}\mathcal {GP}(0, \Theta)$ then follows by setting $\X=\X'$ and $v=u$.

\end{proof}
\subsection{Proof of Proposition \ref{prop:trained_cov}}\label{section:prop_proof_trained_cov}
\trainedcovprop*
\begin{proof}
    We will prove the case for $\sigma^2>0$ as the case for $\sigma^2=0$ is similar, and one can replace inversions of $\Theta(\mathcal X, \mathcal X)$ and $\mathcal K(\mathcal X, \mathcal X)$ with generalised inverses if need be.

    Let $\mathcal X'$ be an arbitrary test set.
    We will first show  $\Sigma_{\text{RP}}\succeq\Sigma_{\text{NNGP}}$. It will suffice to show that $\Sigma_{\text{RP}}(\mathcal X',\mathcal X')-\Sigma_{\text{NNGP}}(\mathcal X',\mathcal X')\succeq 0$ is a p.s.d. matrix. But it is not hard to check that:
    \begin{align}
            \Sigma_{\text{RP}}(\mathcal X',\mathcal X')-\Sigma_{\text{NNGP}}(\mathcal X',\mathcal X') = U(\K(\X, \X) + \sigma^2 I)U^\top
    \end{align}
    which is clearly p.s.d,
    where $U = \Theta(\mathcal X', \mathcal X)(\Theta(\X, \X) + \sigma^2I)^{-1} {-} \mathcal K(\mathcal X', \mathcal X)(\K(\X, \X) + \sigma^2I)^{-1}\in\mathbb{R}^{|\mathcal X'|\times|\mathcal X|}$

    Likewise, to show $\Sigma_{\text{NTK}}\succeq\Sigma_{\text{RP}}$ we can check that:
    \begin{align}\label{eqn:ntk_minus_mix}
        \Sigma_{\text{NTK}}(\mathcal X',\mathcal X')-\Sigma_{\text{RP}}(\mathcal X',\mathcal X') = U_1 + U_2\Delta(\mathcal X, \mathcal X)U_2^\top \succeq 0
    \end{align}
    where
    \begin{align}
    U_1 = \Delta(\mathcal X', \mathcal X') - \Delta(\mathcal X', \mathcal X) \Delta^g(\mathcal X, \mathcal X)\Delta(\mathcal X, \mathcal X')
    \end{align}
    and $\Delta = \Theta^{\leq L}\succeq 0$ is the contributions to the NTK from parameters before the final layer as before. Finally, we need to define $U_2$ as:
    \begin{align}
        U_2 = \Theta(\mathcal X', \mathcal X)(\Theta(\X, \X) + \sigma^2 I )^{-1} - \Delta(\mathcal X', \mathcal X)\Delta^g(\mathcal X, \mathcal X)
    \end{align}
    The notation $\Delta^g(\mathcal X, \mathcal X) $ denotes the generalised inverse. $U_1\succeq 0 $ follows from standard properties of generalised Schur complements, as does the fact that $\Delta(\X', \X)\Delta^g(\X, \X)\Delta(\X, \X)=\Delta(\X', \X)$, which is required for Eq.~(\ref{eqn:ntk_minus_mix}) to hold.

\end{proof}
\section{Alternative constructions of \NTKGP~baselearners}\label{sec:alternative_constructions}
To summarise the analysis in Section \ref{section:correction}, the criteria for an NTKGP baselearner $\tilde f(\cdot, \vtheta)$ is that: \\ 1) $\tilde{f}(\cdot, \vtheta_0)\overset{d}{\rightarrow}\mathcal{GP}(0, \Theta)$ as width increases, while 2) preserving the initial Jacobian $\nabla_{\vtheta}f_0(\cdot)=\nabla_{\vtheta}\tilde f_0(\cdot)$.

A possible alternative construction would be if one could (approximately) sample a fixed $f^*{\overset{d}{\sim}}\mathcal{GP}(0, \Theta)$,
and set:
\begin{align}\label{eqn:delta2}
    \tilde{f}_t(\cdot) = f_t(\cdot) + f^*(\cdot) - f_0(\cdot)
\end{align}
It is easy to approximately sample $f^*$ for finite width NNs using a single JVP, under either standard or NTK parameterisation, by sampling $\tilde{\vtheta}$ independent of $\vtheta_0$ and setting:
\begin{align}\label{eqn:jvp2}
f^*(\vx) = \nabla_{\vtheta}f(\vx, \vtheta_0)\tilde{\vtheta}
\end{align}

Note that \cref{eqn:delta2} requires computation of two forward passes $f_t$ and $f_0$ in addition to a JVP $\nabla_{\vtheta}f(\vx, \vtheta_0)\tilde{\vtheta}$. For some implementations of JVPs, such as in JAX \citep{jax2018github}, the computation of $f_0$ will come essentially for free alongside the computation of $\nabla_{\vtheta}f(\vx, \vtheta_0)\tilde{\vtheta}$, because the JVP is centered about the same ``primal'' parameters $\vtheta_0$
that are used for $f_0$. Hence, this alternative $\tilde f$ presented in \cref{eqn:delta2} may have similar costs to our main construction in \cref{section:correction}, for certain AD packages.

A second valid alternative to $\tilde{f}_t$ would be to replace $f_t$ with $f_t^\text{lin}$, which would give $\tilde{f}^{\text{lin}}(\vx, \vtheta_t) = \nabla_{\vtheta} f(\vx, \tilde{\vtheta})\vtheta_t$ (where we swap $\tilde{\vtheta}$ and $\vtheta_0$ for notational consistency with other NTKGP methods, and initialise at $\vtheta_0$). Because $\tilde{\vtheta}$ is fixed, we see that $\tilde{f}^{\text{lin}}(\cdot, \vtheta_t)$ is linear in $\vtheta_t$. This gives a realisation of the ``sample-then-optimize'' approach \citep{matthewssample} to give posterior samples from randomly initialised linear models, and ensures that $\tilde{f}^{\text{lin}}_{\infty}(\cdot)$ is an exact posterior sample (using the empirical NTK $\hat{\Theta}_0$ as prior kernel) irrespective of parameterisation or width. Note though, of course, the linearised regime holds for $\tilde{f}^{\text{lin}}_t$ throughout parameter space, hence for strongly convex optimisation problems like regression tasks with observation noise, the initialisation is irrelevant. We will call $\tilde{f}^{\text{lin}}_{\infty}$ trained in such a way an \textit{\NTKGP-Lin} baselearner.

\section{Regularisation in the NTKGP and RP training procedures}\label{sec:reg}
As stated in Lemma 3 of \citet{osband2018randomized}, suppose we are in a Bayesian linear regression setting with linear map $g_{\vtheta}(\vz)=\vz^\top\vtheta$, model $y=g_{\vtheta}(\vz) + \epsilon$  for $\eps\sim \mathcal N(0,\sigma^2)$ i.i.d., and parameter prior $\vtheta\sim N(0, \lambda I_p)$. Then, having observed training data $\{(\vz_i, y_i)\}_{i=1}^n$, solving the following optimisation problem returns a posterior sample $\vtheta$:
\begin{align}\label{eqn:sample_then_optimise}
    \tilde \vtheta + \underset{\vtheta}{\text{argmin}}\sum_{i=1}^n \frac{1}{2\sigma^2} \norm{\tilde{y}_i - (g^{\mathstrut}_{\vtheta} + g_{\tilde \vtheta})(\vz_i)}_2^2 + \frac{1}{2\lambda}\norm{\vtheta}_2^2
\end{align}
where $\tilde{y}_i\sim \mathcal N(y_i,\sigma^2)$ and $\tilde \vtheta \sim \mathcal N(0, \lambda I_p)$.

We see that when there is a homoscedastic prior $\mathcal N(0, \lambda I_p)$ for $\vtheta$ that the correct weighting of $L^2$ regularisation is $\norm{\vtheta}_{\Lambda}^2 = \frac{1}{\lambda}\vtheta^\top \vtheta$. In fact, even with a heteroscedastic prior $\vtheta \sim \mathcal N(0, \Lambda)$ with a diagonal matrix $\Lambda \in \mathbb{R}_+^{p\times p}$ and diagonal entries $\{\lambda_j\}_{j=1}^p$, it is straightforward to show that the correct setting of regularisation is $\norm{\vtheta}_{\Lambda}^2 = \vtheta^\top \Lambda^{-1}\vtheta$ in order to obtain a posterior sample of $\vtheta$. For RP-param or NTKGP-param methods, with initial parameters $\vtheta_0$, we have regularisation  $\norm{\vtheta - \vtheta_0}_{\Lambda}^2 = (\vtheta-\vtheta_0)^\top \Lambda^{-1}(\vtheta-\vtheta_0)$, which can be seen as a Mahalanobis distance.

For an NN in the linearised regime \citep{lee2019wide}, this is related to the fact that the NTK and standard parameterisations initialise parameters differently, yet yield the same functional distribution for a randomly initialised NN. In the standard parameterisation, $\lambda_j$ will be a factor of the NN width smaller than in the NTK parameterisation, but the corresponding feature map $\vz$ will be a square root factor of the NN width larger. Thus, solving Eq.~(\ref{eqn:sample_then_optimise}) will lead to the same functional outputs in both parameterisations, if the NN remains in the linearised regime. However, only with our NTKGP trained baselearners $\tilde f$ do you get a posterior interpretation to the trained NN because of the difference between the NNGP and the NTK that standard training does not account for, and because the linearised regime only holds locally to the parameter initialisation.

\section{Additional ensemble algorithms}\label{app:algorithms}
Here, we present our ensemble algorithms for NTKGP-Lin (\cref{alg:ntkgp_lin}) and NTKGP-fn (\cref{alg:ntkgp_fn}), to complement the NTKGP-param algorithm that was presented in \cref{section:algorithms}.
\begin{algorithm}[H]
\caption{\NTKGP-Lin ensemble}
\begin{algorithmic}\label{alg:ntkgp_lin}
\REQUIRE Data $\train=\{\X,\Y\}$, loss function $\mathcal{L}$, NN model $f_{\vtheta}:\mathcal{X}\rightarrow \mathcal{Y}$, Ensemble size $K\in\mathbb{N}$, noise procedure: \texttt{data\_noise}, NN parameter initialisation scheme: \texttt{init}($\cdot$)  \\
\FOR{$k=1,\ldots,K$}
    \STATE Form $\{\mathcal X_k, \mathcal Y_k\} = \texttt{data\_noise}(\mathcal{D})$ \;
    \STATE Initialise $\vtheta_k\overset{d}{\sim} \texttt{init}(\cdot)$ \;
    \STATE Initialise $\tilde{\vtheta}_{k}\overset{d}{\sim} \texttt{init}(\cdot)$  \;
    \STATE Define $\tilde{f}^{\text{lin}}_k(\vx, \vtheta_t) = \nabla_{\vtheta}f(\vx, \tilde \vtheta_k)\vtheta_t$ and set $\vtheta_0=\vtheta_k$
    \STATE Optimise $\mathcal{L}(\tilde{f}^{\text{lin}}_k(\mathcal{X}_k, \vtheta_t), \mathcal{Y}_k) + \frac{1}{2}\norm{\vtheta_t - \vtheta_k}_{\Lambda}^2$ for $\vtheta_t$  to obtain $\hat\vtheta_k$
\ENDFOR
\RETURN ensemble $\{\tilde{f}^{\text{lin}}_k(\cdot,\hat\vtheta_k)\}_{k=1}^K$
\end{algorithmic}
\end{algorithm}

\begin{algorithm}[H]
\caption{\NTKGP-fn ensemble}
\begin{algorithmic}\label{alg:ntkgp_fn}
\REQUIRE Data $\train=\{\X,\Y\}$, loss function $\mathcal{L}$, NN model $f_{\vtheta}:\mathcal{X}\rightarrow \mathcal{Y}$, Ensemble size $K\in\mathbb{N}$, noise procedure: \texttt{data\_noise}, NN parameter initialisation scheme: \texttt{init}($\cdot$)  \\
\FOR{$k=1,\ldots,K$}
    \STATE Form $\{\mathcal{X}_k, \mathcal Y_k\} = \texttt{data\_noise}(\mathcal{D})$\;
    \STATE Initialise $\vtheta_k\overset{d}{\sim} \texttt{init}(\cdot)$\;
    \STATE Initialise $\tilde{\vtheta}_{k}\overset{d}{\sim} \texttt{init}(\cdot)$ and denote $\tilde{\vtheta}_k=\texttt{concat}(\{\tilde{\vtheta}_k^{\leq L}, \tilde{\vtheta}_k^{L+1}\})$ \;
    \STATE  Set $\vtheta_k^* = \texttt{concat}(\{\sqrt{2}\tilde{\vtheta}_k^{\leq L}, \tilde{\vtheta}_k^{L+1}\})$\;

    \STATE Define $\delta(\vx) = \nabla_{\vtheta}f(\vx, \vtheta_k)\vtheta^{*}_{k}$
    \STATE Define $\tilde{f}_k(\vx, \vtheta_t) = f(\vx, \vtheta_t) + \delta(\vx)$ and set $\vtheta_0=\vtheta_k$\;
    \STATE Optimise $\mathcal{L}(\tilde{f}_k(\mathcal{X}_k, \vtheta_t), \mathcal{Y}_k) + \frac{1}{2}\norm{\vtheta_t}_{\Lambda}^2$ for $\vtheta_t$ to obtain $\hat\vtheta_k$
\ENDFOR
\RETURN ensemble $\{\tilde{f}_k(\cdot,\hat \vtheta_k)\}_{k=1}^K$

\end{algorithmic}
\end{algorithm}
In \cref{alg:ntkgp_fn} we seek to reinitialise $\tilde f_k(\vx, \vtheta_0)$ from $\mathcal{GP}(0, \mathcal K)$ to $\mathcal{GP}(0, 2\Theta)$ in the infinite width limit, following the randomised prior function method of \citet{osband2018randomized}. While there are many ways to do this we choose to use only one JVP, with a reweighted tangent vector, for $\delta(\cdot)$ in order to reduce extra computational costs. It would be similarly possible to model a scaling factor $\beta$ for the prior function, like \citep{osband2018randomized}, using a single JVP with a differently reweighted tangent vector.

Note also that for the NTKGP-fn it is unreasonable to assume that the linearised NN dynamics will hold true for the duration of training because, unlike in NTKGP-param (\cref{alg:ntkgp_param}) we regularise towards the origin not the initialised parameters.
\section{Aggregating predictions from ensemble members}\label{app:aggregation}
For completeness, we now describe how to aggregate predictions from ensemble members. Given a test point $(\vx, y)$, for each baselearner NN $k\leq K$, we suppose we have a probabilistic prediction $p_k(y|\vx)$ obtained from the NN output. We then treat the ensemble as a uniformly-weighted mixture model over baselearners and combine predictions as $p(y|\vx)=\frac{1}{K}\sum_{k=1}^K p_k(y|\vx)$. For our Bayesian deep ensembles, we can view this aggregation as a Monte Carlo approximation of the GP posterior predictive with NTK prior.

For classification tasks, this aggregation is exactly an average of predicted probabilities. For regression tasks, the prediction is a mixture of normal distributions, and we follow \citet{balajideepensembles} by approximating the ensembled prediction as a single Gaussian with matched moments. That is to say, if $p_k(y|\vx)\sim \mathcal N(\mu_k(\vx), \sigma^2_k(\vx))$, then we approximate $p(y|\vx)$ by $\mathcal N(\mu_*(\vx), \sigma_*^2(\vx))$ for $\mu_*(\vx)=\frac{1}{K}\sum_k \mu_k(\vx)$ and $\sigma^2_*(\vx)=\frac{1}{K}\sum_k(\mu^2_k(\vx)-\mu^2_*(\vx)) + \sigma^2_k(\vx)$.

\section{Comparison of memory and computation costs for ensemble methods}\label{app:costs}
There is only a negligible training-time computational overhead for our NTKGP methods compared to other ensemble methods \citep{balajideepensembles,osband2018randomized}, for a training set of fixed size (e.g. MNIST, CIFAR-10). This is because one can obtain and store our fixed additive JVPs $\delta$ in a single pass over the training data. For test-time constrained applications, one can employ ensemble distillation \citep{tran2020hydra} for our NTKGP ensembles as one would for standard deep ensembles.

For completeness, we include in Table \ref{tab:costs} (left) the computational cost of different ensemble methods when the modified forward pass $\tilde{f}$ needs to be computed on the fly for new data, though we again stress that this is not necessary for train nor test time, as described in the paragraph above. A rule of thumb for a library offering forward-mode AD, like JAX \citep{jax2018github}, is that a JVP costs on the order of three standard forward passes in terms of FLOPs. We use forward-mode AD to compute JVPs as this is known to be more memory-efficient than reverse-mode AD for JVP computation. It is worth pointing out that our methods share the same trainable parameters as standard deep ensembles, and so do not incur any additional computational cost in the backward pass.

\begin{table*}[ht]
\centering
\caption{Comparison of computational and memory costs of different ensemble methods per ensemble member. Computational costs are specified per (modified) forward pass and represent a naive worst-case scenario (presented for completeness); a more astute approach renders only a negligible difference between ensemble methods, as discussed in this section.}

\begin{tabular}{@{}crrcrr@{}}
\addlinespace
\toprule
 \multicolumn{1}{c}{Method}& \multicolumn{2}{c}{Computational cost} & \phantom{abc} & \multicolumn{2}{c}{Parameter sets to store}\\

\cmidrule{2-3} \cmidrule{5-6}
 & Forward passes & JVPs && Train time & Test time \\ \midrule
Deep ensembles         & 1              & 0   && 1          & 1   \\
RP-param & 1              & 0   && 2          & 1      \\
RP-fn    & 2              & 0   && 2          & 2   \\
\NTKGP-param            & 1              & 1   && 3          & 3   \\
\NTKGP-fn               & 1              & 1   && 3          & 3  \\
\bottomrule
\addlinespace
\end{tabular}
\label{tab:costs}
\end{table*}

In terms of memory, both \NTKGP~and RP methods require storage of extra sets of parameters in order to compute the untrainable additive functions $\delta(\cdot)$ and regularise in parameter space, displayed in Table \ref{tab:costs} (right). However, the activations of the extra forward pass in the Randomised prior function method need not be stored. And moreover, forward mode JVPs are composed alongside the primitive operations that comprise the forward pass, so the memory requirements incurred by the extra JVP are independent of the NN depth for our \NTKGP~methods. Note that the memory bottleneck for large NNs is most often from the need to store activations for the backward pass \citep{gomez_reversible} and not from storing parameter sets, hence our NTKGP ensembles are not affected by the main memory bottleneck for large NNs, relative to standard deep ensembles.

It is worth noting that our Bayesian deep ensembles still retain the distributability of standard deep ensembles. Moreover, our computational and memory costs still scale linearly in dataset size and parameter space dimension, enabling us to work with large scale datasets like Flight Delays \citep{hensman2013gaussian}.

Finally, in this section we only compare the costs associated to different ensemble methods. Ensembles methods are known to be computationally expensive and there has been recent interest in the community to derive new methods \citep{van2020uncertainty, dusenberry2020efficient} that reduce such costs. However, at the time of writing, deep ensembles \citep{balajideepensembles} are state-of-the-art for uncertainty quantification tasks \citep{ovadia2019can}, and hence we believe a comparison of costs between ensemble methods is most appropriate for this work.

\section{Scaling for one-hot targets in classification}\label{app:scaling}
As discussed in \cref{section:classification} and repeated here for completeness: because $\delta(\cdot)$ is untrainable in our NTKGP methods, it is important to match the scale of the NTK $\Theta$ to the scale of the one-hot targets in multi-class classification settings. One can do this either by introducing a scaling factor $\kappa>0$ such that we scale either: 1) $\tilde{f}\leftarrow \frac{1}{\kappa}\tilde{f}$ and so $\Theta\leftarrow \frac{1}{\kappa^2}\Theta$, or 2) $e_c\leftarrow \kappa e_c$ where $e_c\in\mathbb{R}^C$ is the one-hot vector denoting class $c\leq C$. We choose option 2) for our implementation.

To set $\kappa$, for each ensemble method we calculated the mean squared values of baselearner outputs at initialisation, which we define for convenience as $\zeta_0$, on the training set for that particular ensemble method, and tuned $\kappa^2$ (based on validation accuracy) on a small linear scale centered around $C\zeta_0$, where $C$ is the number of classes. This is in order to match the second moments of the random NNs at initialisation with the scaled one-hot targets across the $C$ classes. For example, for NTKGP-param, we set $\zeta_0=\frac{1}{|\X|}\sum_{\vx\in \X}\Theta(\vx, \vx)\in \mathbb{R}^+$.

To illustrate the importance of $\kappa$, in \cref{fig:cifar10_absolute_scale} we present the corresponding results to \cref{fig:cifar10} where instead of setting $\kappa$ dependent on the scale of each ensemble methods' initialised baselearners, as above, we set $\kappa=\frac{1}{|\X|}\sum_{\vx\in \X}\Theta(\vx, \vx){\in} \mathbb{R}$ for all ensemble methods. This is the base $\kappa$ value for NTKGP-param  at initialisation, but note that we did not tune neither $\kappa$ (around this base value) nor weight variance (set at $\sigma_W^2=2$ like He initialisation \cite{He_2015_ICCV}, which has been optimised for standard NNs and hence standard deep ensembles) for \cref{fig:cifar10_absolute_scale}.
    \begin{figure}[h]
\centering
    \hspace*{\fill}%
    \begin{subfigure}{0.35\textwidth}
      \centering
          \includegraphics[width = \linewidth]{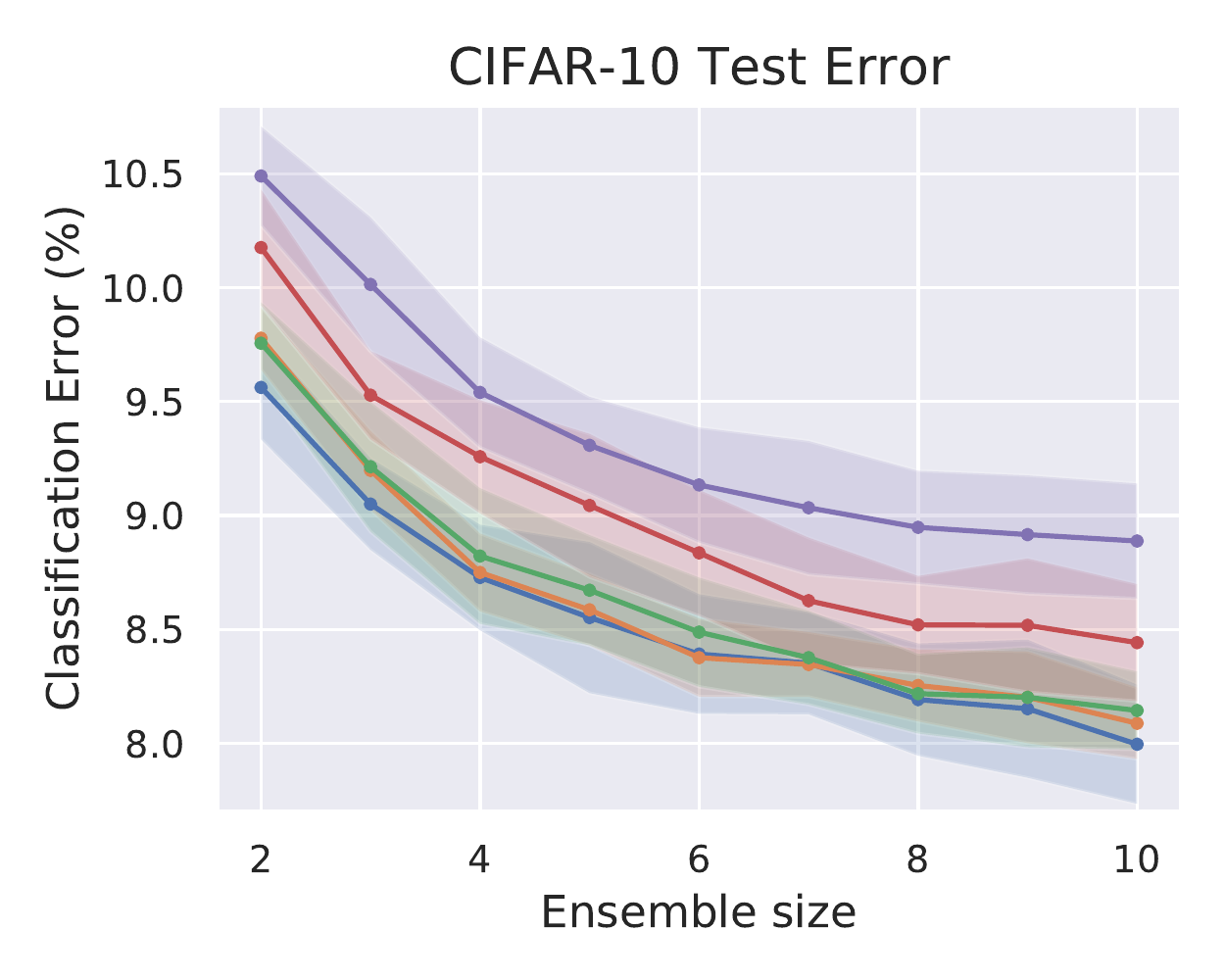}
    \end{subfigure}
    \hfill%
    \begin{subfigure}{0.35\textwidth}
      \centering
      \includegraphics[width = \linewidth]{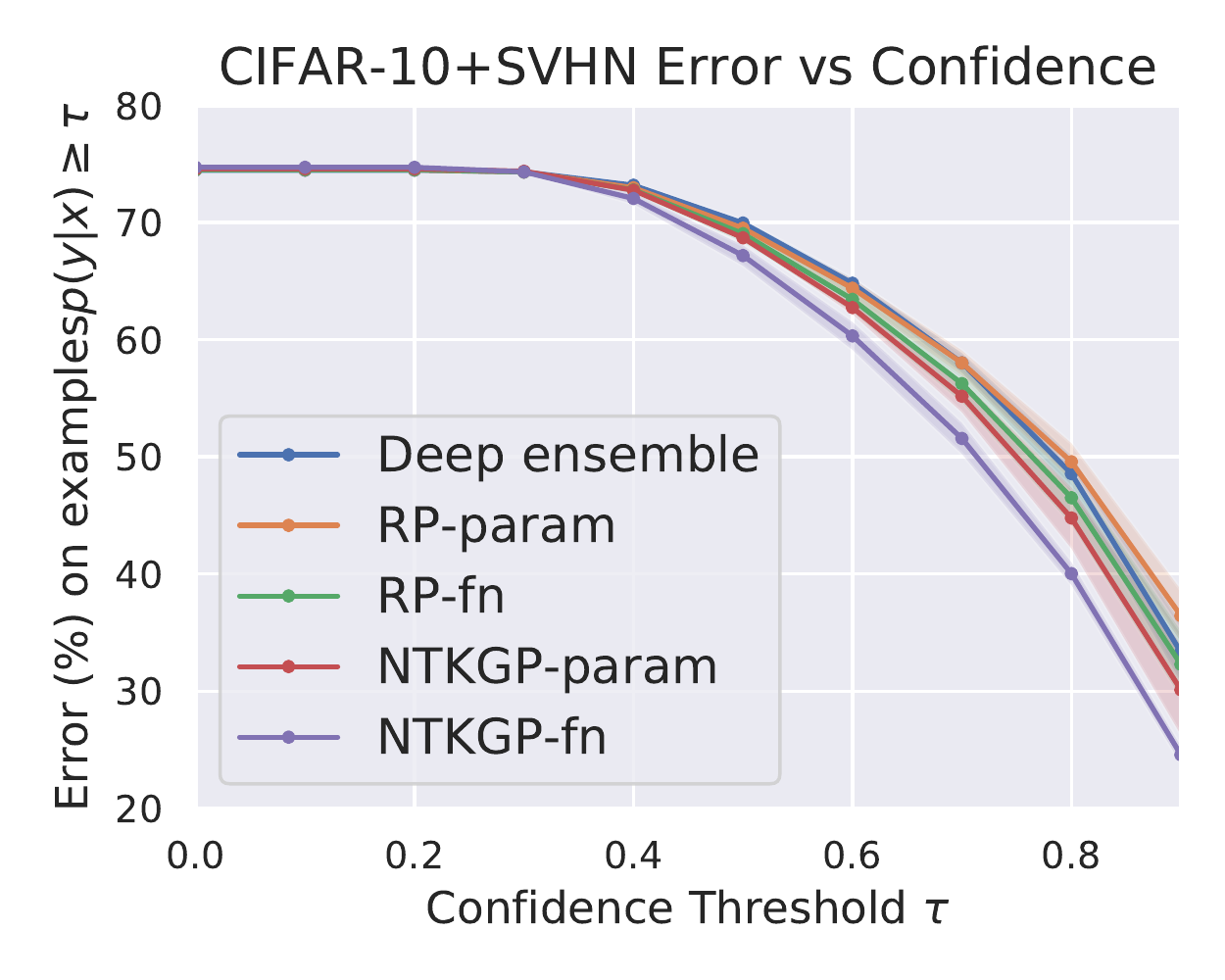}
    \end{subfigure}
    \hspace*{\fill}%
        \reducespaceafterfigure
    \caption{\cref{fig:cifar10} but where regression target scale $\kappa$ is constant across ensemble methods and set to match the second moment of the NTK on the training set at initialisation. Error bars correspond to 5 independent runs.}
    \label{fig:cifar10_absolute_scale}
\end{figure}

In \cref{fig:cifar10_absolute_scale} we see a different results to \cref{fig:cifar10}, as here our NTKGP methods suffer slightly on in-distribution performance but also outperform the baselines methods on out-of-distribution detection. This highlights the importance of the regression target scale when considering classification tasks, and moreover reflects a general theme in our experiments of the trade-off between more aggressive predictions (that tend to perform better on in-distribution) and more conservative predictions (that tend to perform better on out-of-distribution). In our classification methodology, larger $\kappa$ values tend to lead to more confident predictions. We point out that this is an issue that affects all ensemble methods and is not limited to our Bayesian ensembles.

\section{Experimental Details \& additional plots}\label{app:experimental_details}
\subsection{Toy 1d example}
We set ensemble size $K=20$, and train on full batch GD with learning rate $0.001$ for 50,000 iterations under standard parameterisation in Neural Tangents \citep{neuraltangents2020}, with $\sigma_W=1.5$ \& $\sigma_b=0.05$, for $\sigma_W, \sigma_b$ defined as in \cref{app:parameterisations}. In \cref{fig:1d_ensemble_size} we evaluate the impact of the ensemble size on this toy problem for different ensemble methods. We find that, of the two methods that approximate the analytic NTKGP mean predictor (c.f. \cref{table:predcomp}), the approximation of the analytic mean predictor for NTKGP-param degrades compared to RP-param at small ensemble sizes, although the predictive uncertainties
are well matched even at small ensemble sizes. The degradation in mean predictor is unsurprising as there is more (untrainable) noise in the initialised NTKGP baselearners. One simple possible solution to this problem, which we leave for future work, is to use separate baselearners for the mean and uncertainty predictions, like in \citet{Ciosek2020Conservative}.

\begin{figure}[h]
  \centering
      \includegraphics[width=\textwidth]{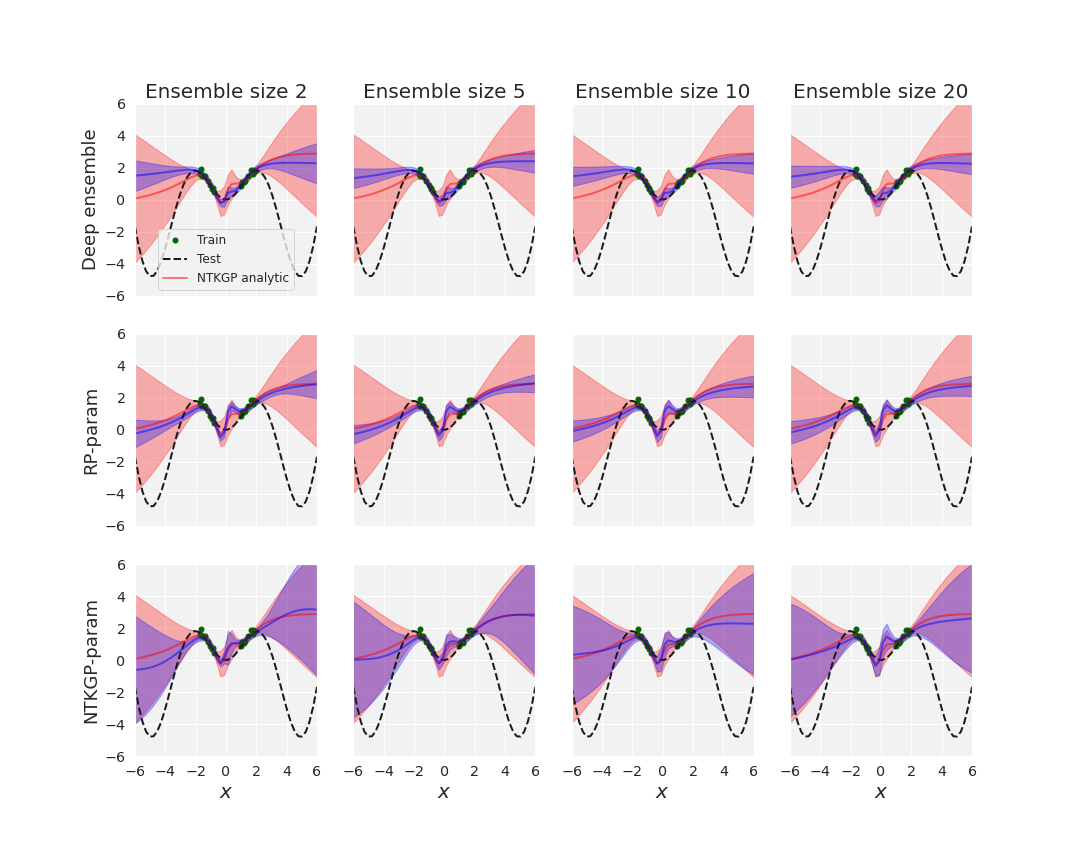}
      \vspace{-1.em}
  \caption{Comparison between ensemble methods (blue) and the analytic NTKGP posterior as the ensemble size is varied on toy example.}
  \label{fig:1d_ensemble_size}
\end{figure}

\subsection{Flight Delays}
Our baselearners are MLPs with 4 hidden layers, 100 hidden units per layer and ReLU activations, and we use standard parameterisation with $\sigma_W=1$ \& $\sigma_b=0.05$, and choose ensemble size $K=5$. We train for 10 epochs with learning rate 0.001, batch size 100 and Adam \citep{kingma2014adam}. For all experiments, all ensemble methods apart from standard deep ensembles \cite{balajideepensembles} are $L^2$ regularised according to Appendix \ref{sec:reg}, with weight decay strength set to $10^{-4}$ for standard deep ensembles.

We use a validation set of size 50k that is sampled uniformly from the training set of size 700k, and early stop baselearner NNs based on validation set loss. Inputs and targets are standardised so that the training data is zero mean and unit variance.

\subsection{MNIST vs. NotMNIST}
For all image classification experiments, we use a $90{-}10\%$ split for the train-validation sets needed for temperature scaling.

Baselearners are MLPs with 2-hidden layers, 200 hidden units per layer and ReLU activations. We standardise data to have mean $0$ and standard deviation $1$
across flattened pixels.

For all ensemble methods, we use standard parameterisation with fixed bias standard deviation $\sigma_b=0.05$, observation noise $\sigma=0.1$ and tune weight variance $\sigma_W^2$ on a small linear scale around $\sigma_W^2=2$. We set observation noise $\sigma=0.1$ for . We train for 20 epochs with batch size 100, learning rate 0.001 and Adam \citep{kingma2014adam}. We do not early stop for any classification experiment, and use the final trained baselearners throughout.

For the analytic NTKGP results, we use the NTK in NTK parameterisation, and use the same observation noise and bias variance as for ensemble methods. However, we fix $\sigma_W^2=2$ and also do not tune target scale $\kappa$ (set to the base value described in \cref{app:parameterisations}) due to computational resources. We also use only half the test sets both for MNIST and NotMNIST due to resource requirements, keeping the ratios of test sizes consistent in order for the error versus confidence plot \cref{fig:mnist} (right) to be comparable. To compute test and out-of-distribution predictions, having obtained the optimal temperature scale $T^*$ and  analytic NTKGP predictions in logit space, $p(\cdot|\X,\Y)$, we approximate the softmax class probability predictions: $\int \text{softmax}(z/T^*)p(dz|\X,\Y)$, by a Monte Carlo ensemble approximation with 100 samples.

 For all classification ensemble  methods, we temperature scale on validation cross entropy for 5 epochs with batch size 100 and learning rate 0.1, whereas for analytic NTKGP we temperature scale for 1000 epochs on full batch size 6000. Like above, we approximate the analytic NTKGP validation predictions (for temperature scaling) by a Monte Carlo ensemble, this time of size 10. We found the various temperature scaling training hyperparameter considerations here to be unimportant to achieve convergence, due to the fact that the temperature scale is a scalar value.
\subsection{CIFAR-10 vs SVHN}
Baselearners are Myrtle-10 CNNs \citep{shankar2020neural} with 100 channel width and ReLU activations. We use  $\sigma_b=0.01$ and set observation noise $\sigma=0.1$. Like for MNIST we tune $\sigma^2_W$ on a small linear scale around $\sigma^2_W=2$. We train using SGD, with momentum parameter 0.9, for 100 epochs and learning rate $0.001$, which is decayed to $0.0002$ after 80 epochs. In the first 5 epochs we raise the learning rate in linear increments from $0.0001$ to $0.001$. We use batch size 125. During training we apply random crops and horizontal flips before standardisation. We do not compare to the analytic NTKGP for the Myrtle-10 CNN due to resource requirements.

\cref{fig:cifar10_vs_svhn_entropies} displays entropy histograms for ensembles trained on CIFAR-10 and tested on in distribution CIFAR-10 test data and out-of-distribution SVHN test data, corresponding to the same experiments as in \cref{fig:cifar10}. As we can see, there is a much less noticeable difference between ensemble methods compared to the simpler MNIST vs NotMNIST case.
\begin{figure}[h]
    \centering
    \begin{subfigure}[b]{0.8\textwidth}
        \includegraphics[width = \linewidth]{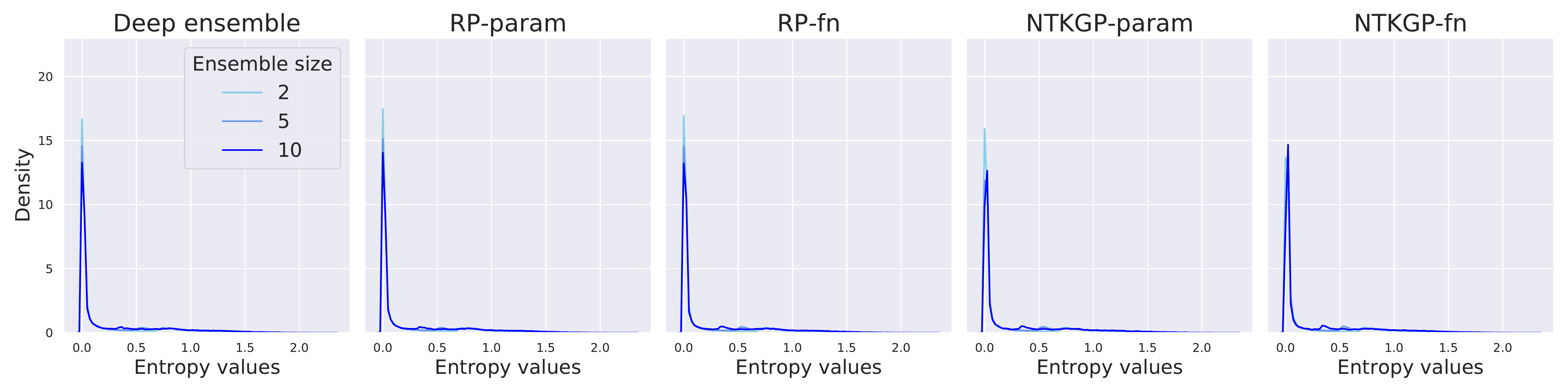}
    \end{subfigure}

    \begin{subfigure}[b]{0.8\textwidth}
        \includegraphics[width = \linewidth]{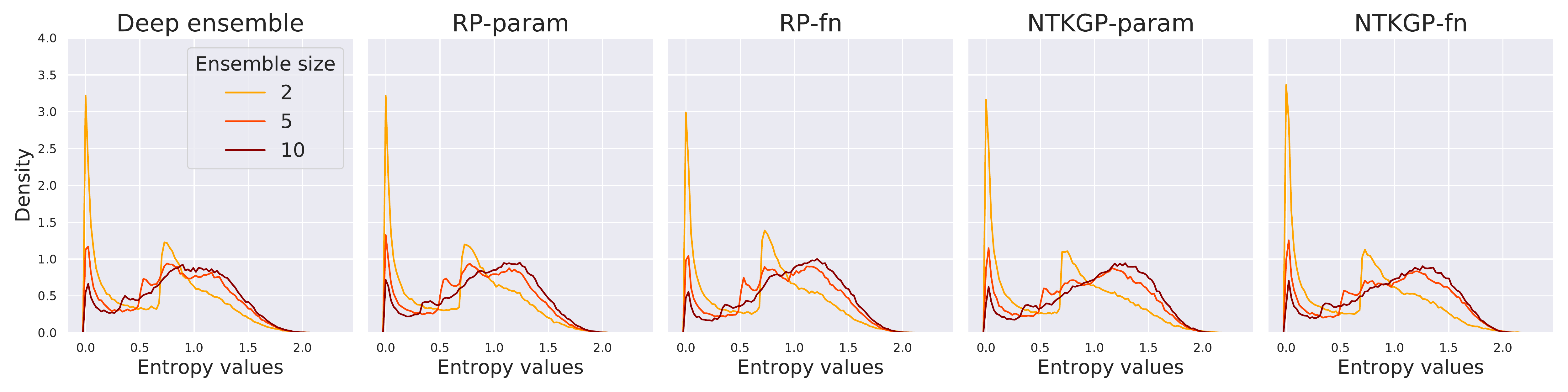}
    \end{subfigure}
    \caption{Histograms of predictive entropy on CIFAR-10 (top) and SVHN (bottom) test sets for different ensemble methods and for different ensemble sizes.}
    \label{fig:cifar10_vs_svhn_entropies}
\end{figure}
\end{document}